%% file: arxiv.tex
\documentclass[11pt]{article}

\usepackage[margin=1in]{geometry}
\usepackage[T1]{fontenc}
\usepackage[utf8]{inputenc}
\usepackage{lmodern}          
\usepackage{authblk}
\usepackage{microtype}        

\usepackage{amsmath,amssymb,amsthm}

\usepackage[round]{natbib}

\usepackage{color-edits}
\addauthor{jw}{blue}

\input{preamble}

\title{Policy Learning with Abstention}

\author{Ayush Sawarni}
\author{Jikai Jin}
\author{Justin Whitehouse}
\author{Vasilis Syrgkanis}

\affil{Stanford University\\
\texttt{\{ayushsaw,jkjin,jwhiteho,vsyrgk\}@stanford.edu}}

\begin{document}
\maketitle

\begin{abstract}
Policy learning algorithms are regularly leveraged in domains such as personalized medicine and advertising
to develop individualized treatment regimes. However, a critical deficit of existing algorithms is that they force a decision even when predictions are uncertain, a risky approach in high-stakes settings. The ability to abstain --- that is, to defer to a safe default or an expert --- is crucial but largely unexplored in this context.
To remedy this, we introduce a framework for \textit{policy learning with abstention}, in which policies that choose not to assign a treatment to some customers/patients receive a small, additive reward on top of the value of a random guess.  We propose a two–stage learner that first identifies a set of near–optimal policies and then constructs an abstention class based on disagreements between the policies. We establish fast $O(1/n)$–type regret guarantees for the abstaining policy when propensities are known, and show how to extend these guarantees to the unknown–propensity case via a doubly robust (DR) objective. Furthermore, we demonstrate that our abstention framework is a versatile tool with direct applications to several other core problems in policy learning. We use our algorithm as a black box to obtain improved guarantees under margin conditions without the common realizability assumption. We also show that abstention provides a natural connection to both distributionally robust policy learning, where it acts as a hedge against small data shifts, and safe policy improvement, where the goal is to improve upon a baseline policy with high probability.
\end{abstract}

\input{sections/introduction}
\input{sections/preliminaries}

\input{sections/learning_with_abstention}

\input{sections/applications}
\input{sections/conclusion}
\section{Acknowledgments}
Vasilis Syrgkanis is supported by NSF Award IIS-2337916. Ayush Sawarni is partially supported by
NSF Award IIS-2337916.
\bibliography{references}

\clearpage
\appendix

\input{sections/appendix-main}

\input{sections/appendix-dr}
\input{sections/appendix-applications}

\input{sections/appendix-experiments}

\end{document}

%% file: sections/introduction.tex
\section{Introduction}

Policy learning algorithms guide high-stakes decisions in domains like personalized medicine and public policy by constructing individualized treatment rules from observational data. Yet, a critical deficit of existing methods is their failure to abstain when faced with high uncertainty. In safety-critical applications, forcing a decision when evidence is weak can be harmful; the responsible action may be to defer to a human expert or a trusted baseline policy. One way to facilitate this is through allowing learned policies to \textit{abstain} --- that is, instead of just assigning a unit to treatment or control (denoted ``1'' or ``0'' respectively), they can abstain by outputting a special symbol ``$\ast$''.  

While abstaining algorithms have been developed in the classification literature \citep{bousquet2021fast}, generalizing these algorithms to the setting of policy learning from observational data is a non-trivial task. The primary difficulty is the counterfactual nature of the problem: the learner only ever observes an outcome under treatment or control for any given unit, never both. Because of this, the learner must transform samples (either via inverse propensity weighting or through a doubly-robust correction) into ``pseudo-outcomes'' --- quantities representing the unit-specific contrasts in outcomes between treatment and control. 



In this work, we present a framework formalizing policy learning with abstention. The goal of the learner is to use observational data $Z = (X, D, Y)$ consisting of covariates $X \in \calX$, treatment $D \in \{0, 1\}$, and outcome $Y \in \R$ to learn a high reward \textit{abstaining treatment policy} $\pi : \calX \rightarrow \{0, 1, \ast\}$. Here, $\pi$ is assumed to exists in some pre-defined policy class $\Pi$ that implicitly encodes constraints of the problem (e.g.\ $\Pi$ is the class of all depth 3 decision trees, thus enforcing explainability). Our per-section contributions are as follows:

\begin{enumerate}
    \item In Section~\ref{sec:abstention}, we formalize the problem of policy learning with abstention. We introduce a form of ``safe'' or abstaining regret, which incentives abstention in regions of uncertainty by offering a small, additive reward $p$ on top of the value of a random guess. 
    \item Also in Section~\ref{sec:abstention}, we introduce our primary algorithm for policy learning with abstention (Algorithm~\ref{algo:abs}). Our algorithm first de-biases observations via inverse propensity weighting (IPW) when propensities are known (Section~\ref{sec:known_prop}), or by performing a \textit{doubly-robust} correction when propensities are unknown (Section~\ref{sec:unknown_prop}). We prove the algorithm obtains fast $O(1/n)$ high probability regret rates under the safe/abstaining regret, and in the worst-case obtains $O(1/\sqrt{n})$ standard regret, matching rates due to \citet{athey2021policy}.
    \item In Section~\ref{sec:safe_policy}, we discuss applications of Algorithm~\ref{algo:abs} to a variety of practically-relevant policy learning problems. This includes the problem of policy learning without standard margin assumptions (Theorems~\ref{thm:finite-D} and \ref{thm:reg-oracle}), safe-policy improvement (Algorithm~\ref{algo:safe_policy_abstention}), and learning under distribution shift (Proposition~\ref{prop:dist-shift}).

\end{enumerate}
In addition to the above, we also conduct simulations to compare our safe policy improvement algorithm (Algorithm~\ref{algo:safe_policy_abstention}) to relevant benchmarks \citep{thomas2015high}. In sum, our work illustrates the practicality and importance of studying policy learning with abstention, and further provides insight into how abstention can impact other areas of policy learning as well.

\paragraph{Related work on abstention.}
Classification with abstention/a reject option traces back to \citet{chow1970optimum}, who characterize the optimal error–reject trade-off, prescribing abstention whenever the posterior risk of mis-classification exceeds the reject cost. \citet{herbei2006reject} consider classification with a reject option in a statistical learning setting, and \citet{bartlett2008reject} introduce a convex ``reject-hinge'' loss to make abstention trainable in large-scale settings. Other researchers have developed abstaining classification algorithms for SVMs~\citep{grandvalet2008svmreject}, multiclass classification~\citep{ramaswamy2015consistent}, boosting~\citep{cortes2016boosting}, online learning ~\citep{cortes2018online, pasteris2024bandits,neu2020fast, vanderHoeven2022regretvariance}, and even deep learning settings~\citep{geifman2017selective, geifman2019selectivenet}. \citet{madras2018defer} study a setting where the prediction of an abstaining algorithm is replaced by a downstream expert, but don't provide regret/excess risk guarantees.
Most closely related to this paper is the work of \citet{bousquet2021fast}, who develop an empirical risk minimization-based algorithm for abstention that enables fast $O(1/n)$ risk minimization rates without standard margin/realizability assumptions.


\paragraph{Policy Learning.} Our work builds off of the literature on policy learning, particularly works on welfare maximization/regret minimization~\citep{athey2021policy, hirano2009asymptotics, kitagawa2018should, manski2004statistical}.
Also related to our work is the literature of safe policy improvement, which aims to produce a policy that, with high-probability, improves over some baseline ``safe'' policy. Existing approaches include using hypothesis testing to select between the baseline and learned policy~\citep{thomas2015high, cho2025cspi},  bootstrap the baseline policy with a learned policy in regions of low confidence~\citep{laroche2019safe, simao2019safe}, and directly minimizing ``negative regret'' \citep{ghavamzadeh2016safe}. Our work provides new perspective onto safe policy improvement, showing how abstaining algorithms can naturally be used to construct improved policies.
Other complementary areas that are not directly considered in the work are policy evaluation~\citep{dudik2011doubly, karampatziakis2021off}, offline reinforcement learning~\citep{moodie2007demystifying, moodie2012q, murphy2003optimal}, policy learning under partial identification~\citep{kallus2018confounding}, and inference on values of optimal treatment policies~\citep{luedtke2016statistical, chen2023inference, whitehouse2025inference}.



%% file: sections/preliminaries.tex
\section{Notation and Preliminaries} \label{sec:prelims}
We assume the learner observes i.i.d.\ draws $Z=(X,D,Y)$ with covariates $X\in\cX$, realized action $D\in\{0,1\}$, and outcome $Y\in[0,1]$. Let $(Y(1),Y(0))$ denote potential outcomes, and write
\[
g_o(d,x)\coloneqq \E\!\big[Y(d)\mid X=x\big],
~~~
\tau_o(x)\coloneqq g_o(1,x)-g_o(0,x),
\]
as respectively the expected outcome mapping and conditional average treatment effect (CATE).
We denote the propensity score by $p_o(x)\coloneqq \prob(D{=}1\mid X{=}x)$.  We assume the following of the data generating process.

\begin{assumption}[Data-generating process.]\label{assump:data}
We assume the following conditions hold throughout:
\begin{enumerate}
\item[(i)] \emph{Unconfoundedness:} $(Y(1),Y(0))\perp D \mid X$.
\item[(ii)] \emph{Strict overlap:} $p_o(x)\in[\kappa,1-\kappa]$ for some known $\kappa\in(0,\tfrac12]$.
\end{enumerate}
\end{assumption}

We let $\hat g, \hat p, \hat \tau$ denote generic estimates of $g_o, p_o , \tau_o$, respectively. For $q\ge 1$ and a distribution $P_X$ on $\cX$, define
$\|f\|_{P_X,q}\coloneqq\big(\E_{P_X}|f(X)|^q\big)^{1/q}.$

\paragraph{Policies.}
A policy $\pi:\cX\to\{0,1\}$ prescribes a binary treatment to a unit with covariates $X$. Let $\Pi$ denote a class of feasible treatment policies. We assume the $\Pi$ has bounded complexity, controlled through the VC dimension~\citep{vapnik2013nature}.
\begin{assumption}[Policy class complexity]
\label{assump:vc}
The policy class $\Pi$ has finite VC dimension $d<\infty$.\
\end{assumption}
The in-class optimal policy is $
\pi^*\in\argmax_{\pi\in\Pi} V(\pi)$,
where $V(\pi) := \E_{P_X}[Y(\pi(X))]$ denotes the expected welfare/reward under policy $\pi$.
We also consider the \textit{Bayes policy}, which maximizes value over all measurable binary policies:  $\pi^B(x)\;\coloneqq\;\argmax_{d\in\{0,1\}} \E\big[Y(d)\mid X=x\big]
\;=\;\indic{\tau_o(x)\ge 0}$.

\paragraph{Policy value.}
For binary policies, we can rewrite the policy value $\E_{P_X}[Y(\pi(X))]$ as
\begin{align*}
V(\pi)
&\coloneqq \E\Big[\pi(X)\,Y(1)+\big(1-\pi(X)\big)\,Y(0)\Big] \\
& = \E\Big[\pi(X)\tfrac{YD}{p_o(X)} + (1 - \pi(X))\tfrac{Y(1 - D)}{1 - p_o(X)}\Big] \nonumber.
\label{eq:V-binary}
\end{align*}

Likewise, we define the conditional value for policy $\pi$ at realized covariates $x$ as
\begin{align}
v(\pi,x)
&\coloneqq \E\Big[\pi(X)\,Y(1)+\big(1-\pi(X)\big)\,Y(0)\,\Big|\,X=x\Big]\nonumber\\
&=\pi(x)\,g_o(1,x)+\big(1-\pi(x)\big)\,g_o(0,x).
\end{align}
Let $\E_n$ denote the expectation with respect to the empirical sample distribution. We define the inverse propensity weighted (IPW) analogue of Equation~\eqref{eq:V-binary} as
\begin{equation*}
V_n(\pi)
\;\coloneqq\;
\E_n\Big[\pi(X)\,\tfrac{YD}{p_o(X)}+\big(1-\pi(X)\big)\,\tfrac{Y(1-D)}{1-p_o(X)}\Big].
\end{equation*}

Additionally, we define the normalized IPW contribution for a policy $\pi$ at an observation $z=(x,d,y)$:
\begin{equation*}
f_\pi(x,d,y)\;\coloneqq\;\kappa\!\left(\pi(x)\,\frac{y d}{p_o(x)}+\big(1-\pi(x)\big)\frac{y(1-d)}{1-p_o(x)}\right).
\end{equation*}
This quantity is the per-sample term whose average recovers the policy value up to the factor $\kappa$; the extra $\kappa$ factor ensures the function is uniformly bounded. We use $f(z)$ and $f(x,d,y)$ interchangeably. By $Y\in[0,1]$ and $p_o(X)\in[\kappa,1-\kappa]$, we have $f_\pi(Z)\in[0,1]$ a.s., and
$
\E\!\big[f_\pi(Z)\big]=\kappa\,V(\pi).
$
As a notational shorthand,  write $\E[f_\pi]$ and $\E_n[f_\pi]$ for $\E[f_\pi(Z)]$ and $\E_n[f_\pi(Z)]$, respectively. We use this notation throughout. In particular, when we write
$\E_n \abs{f_{\hat\pi}-f_\pi}$,
we mean the empirical average of pointwise differences between the normalized IPW scores:
$
\tfrac{1}{n}\sum_{i=1}^n \abs{\,f_{\hat\pi}(x_i,d_i,y_i)-f_{\pi}(x_i,d_i,y_i)\,}.
$

%% file: sections/learning_with_abstention.tex
\section{Policy Learning with Abstention} \label{sec:abstention}

We now present our framework for \textit{policy learning with abstention}. In this framework, the learner can choose to defer on recommending a treatment to any unit based on their observed covariates $X$. For instance, the learner may not want to prescribe treatment when the estimated CATE for a unit is small. When the learner abstains, they receive a small, additive reward over the value of a random guess. 

More formally, an \textit{abstaining policy} is a mapping $\pi : \calX \rightarrow \{0,1,*\}$, where $\pi(X) = *$ denotes the deferral of treatment decision. When the learner abstains, they receive reward $\tfrac{Y(1)+Y(0)}{2}+p$, where $p \geq 0$ is some fixed bonus. We define the expected reward mapping under abstention as
\begin{equation*}
g_o(*,x) \coloneqq  \E\Big[ \tfrac{Y(1)+Y(0)}{2} + p \,\Big|\, X=x\Big].
\end{equation*}
Likewise, we define the expected reward/welfare $V^{(p)}(\pi)$ of an abstaining policy with bonus $p$ (and its empirical analogue $V_n^{(p)}(\pi)$) as
{\small
\begin{align*}
V^{(p)}(\pi)
&\coloneqq \E\!\Big[\indic{\pi(X)\neq *}\,v(\pi, X)\;+\;\indic{\pi(X)=*}\, g_o(*,X)\Big]. \\
 V^{(p)}_n(\pi)
&\coloneqq \E_n\Big[\pi(X)\,\tfrac{YD}{p_o(X)}+\big(1-\pi(X)\big)\,\tfrac{Y(1-D)}{1-p_o(X)} \nonumber\\[-0.25em]
&\quad+\;\indic{\pi(X)=*}\big(\tfrac{YD}{2 p_o(X)}+\tfrac{Y(1-D)}{2 (1-p_o(X)) }+p\big)\Big].
\end{align*}}
Note that we have $V^{(p)}(\pi) = V(\pi)$ and $V_n^{(p)}(\pi) = V_n(\pi)$ for any policy $\pi$ that does not abstain. 

The goal of the learner is to use either experimental (Section~\ref{sec:known_prop}) or observational (Section~\ref{sec:unknown_prop}) data to learn an abstaining policy with small \textit{abstaining regret}, which is defined with respect to a binary policy class $\Pi$ as 
\begin{equation}
\Reg_n^{(p)}(\pi) := V(\pi^\ast) - V^{(p)}(\pi),
\label{eq:regret}
\end{equation}
where again $\pi^\ast$ is the in-class optimal policy and $\pi$ is some potentially abstaining policy. We define the \textit{classical regret} just as $\Reg_n(\pi) := \Reg_n^{(0)}(\pi)$.
The learner is constrained to returning a policy $\pitil$ that aligns with some $\pi \in \Pi$ when it does not abstain,  i.e.\ $\pitil(x) = \pi(x)$ when $\pitil(x) \neq *$. 

\subsection{Known Propensities}
\label{sec:known_prop}
We first describe an algorithm for learning an abstaining policy when the treatment assignment mechanism (i.e.\ propensity) for the data is known. Our algorithm (Algorithm~\ref{algo:abs}) works by first computing an empirical welfare maximizer $\wh{\pi}$ from $\Pi$ using half of the data (where welfare is determined with respect to the IPW empirical value, $V_n$). Then, it computes a set of ``near-optimal'' policies whose empirical welfare is close to that of the empirical risk minimizer,  modifying these policies to abstain precisely when they disagree  $\wh{\pi}$. Finally, the algorithm returns the empirical abstaining welfare maximizer in this restricted set of policies on the second half of the data. The following theorem shows that Algorithm~\ref{algo:abs} obtain fast abstaining regret rates, per the formulation of regret in Equation~\eqref{eq:regret}.

\begin{theorem}
\label{thm:abstention-rate}
Fix $p>0$ and $\delta\in(0,1)$. Under Assumption \ref{assump:data} and Assumption \ref{assump:vc} and the construction in Algorithm \ref{algo:abs}, with probability at least $1-\delta$,
\[
 \Reg_n^{(p)}(\pitil) := V(\pist)\;-\;V^{(p)}(\,\pitil\,)
 \;\lesssim\;
 \frac{ d\,\log\!\frac{n}{d} \;+\; \log\!\frac{1}{\delta} }{ p\,n\,\kappa^2 }.
\]
\end{theorem}
\noindent More details and the full proof are provided in Appendix~\cref{appendix:main}.

. 
\begin{remark}
    Theorem \ref{thm:abstention-rate} analyzes policy learning with an abstention bonus and establishes an $O(1/n)$ fast rate when performance is measured against the best \emph{binary} policy in the class. This benchmark is standard in the abstention literature (see, \emph{e.g.}, \cite[Section 1.2]{bousquet2021fast}). In Section \ref{subsec:distribution-robust}, we will discuss a natural connection between this regret bound and distributional robustness.
\end{remark}

We now provide some intuition for the proof. Typically, to obtain fast regret rates for policy learning, there must exist a \textit{margin} on the CATE, i.e.\ the existence of some value $ h > 0$ such that
\begin{equation}
\P(|\tau_o(X)| \geq h) = 1.
\label{eq:margin}
\end{equation}
The abstention bonus $p$ can be viewed as a ``synthetic'' margin --- if the learner abstains when $|\tau_o(X)| < p$, they automatically accumulate higher reward than any binary treatment assignment policy. Likewise, when $|\tau_o(X)| \geq p$, the learner will be more certain in their decisions, and hence unlikely to abstain. While this intuition is just heuristic, we provide a fully rigorous proof in Appendix~\ref{appendix:main}. One important thing to note is that, in Algorithm~\ref{algo:abs}, we never actually need to estimate the CATE. This is of particular importance when the CATE may be a highly complicated function and is thus difficult to capture using statistical learning methods.


When the additive bonus is $p = 0$, our algorithm can no longer be expected to obtain fast $O(1/n)$ rates. The following proposition shows that, in this undiscounted setting,  Algorithm~\ref{algo:abs} still obtains $O(1/\sqrt{n})$ regret rates outlined, which are known to be generally unimprovable without margin or realizability assumption~\citep{athey2021policy}. In this setting, we can convert an abstaining policy into a binary one by having it assign a treatment uniformly at random in the set $\{0,1\}$. The following proposition provides this slower regret rate, showing there is no added risk of using our algorithm over just empirical welfare maximization even in setting without the additive bonus. We prove this result in Appendix~\ref{appendix:main}.

\begin{proposition}[ERM benchmark at $p=0$]
\label{prop:p-zero-benchmark}
Let $\pi_p$ be the output of Algorithm \ref{algo:abs} for any fixed $p\in(0,1)$. Consider its value under $p{=}0$, denoted $V^{0}(\cdot)$. Then, with probability at least $1-\delta$,
\[
 \Reg_n(\pi_p) := V(\pist)-V^{0}(\pi_p)
 \lesssim
 \frac{1}{\kappa}\,
 \sqrt{ \frac{ d\,\log\!\frac{n}{d} \;+\; \log\!\frac{1}{\delta} }{ n } }.
\]
\end{proposition}
\noindent The proof is provided in Appendix~\cref{appendix:main}.

\begin{algorithm}[t]
\caption{Policy Learning with Abstention}
\label{algo:abs}
\begin{small}
\begin{algorithmic}[1]
\STATE \textbf{Input:} Samples $\{(X_i,D_i,Y_i)\}_{i=1}^n$, policy class $\Pi$, overlap $\kappa$, confidence $\delta\in(0,1)$, bonus $p$, VC dimension $d$
Set $\displaystyle
\alpha \leftarrow \sqrt{\tfrac{d \log\!\frac{n}{d} + \log\!\frac{1}{\delta}}{n}}$.
\STATE \textbf{Split:} Partition the samples into two sets of size $n/2$: $\mathcal{D}_1, \mathcal{D}_2$. 
\STATE \textbf{EWM:} $\displaystyle
\pihat = \mathop{\mathrm{arg\,max}}_{\pi \in \Pi} \; V_n(\pi)$ \quad (computed on $\mathcal{D}_1$).\label{step_abs:first_ewm}
\STATE \textbf{Select Near-optimal policies:}
\vspace{-0.7em}
\[{\small
\widehat{\Pi} \!\leftarrow \!\Big\{ \pi \in \Pi \!:
V_n(\pihat) - V_n(\pi)
\le \frac{c}{\kappa}\Big(\alpha^2 + \alpha \sqrt{\E_n \abs{ f_{\pihat} - f_{\pi} }}\Big) \Big\}.}
\] \label{step_abs:almost_optimal}
\vspace{-1.2em}
\STATE \textbf{Abstention projection:} For each $\pi \in \widehat{\Pi}$, define
\[
{\pi'}(X) =
\begin{cases}
\pi(X), & \text{if } \pi(X) = \pihat(X),\\
*, & \text{otherwise},
\end{cases}
\quad
\widetilde{\Pi} \leftarrow \{\, {\pi'} : \pi \in \widehat{\Pi} \,\}.
\] \label{step_abs:abs_projection}
\STATE \textbf{EWM with abstention:} $\displaystyle
\pitil \in \mathop{\mathrm{arg\,max}}_{{\pi}\in \widetilde{\Pi}} \; V^{(p)}_{n}({\pi})$
\quad (evaluate $V^{(p)}_{n}$ on $\mathcal{D}_2$).\label{step_abs:second_ewm}
\STATE \textbf{Return} $\pitil$.
\end{algorithmic}
\end{small}
\end{algorithm}

\subsection{Unknown Propensities: Doubly Robust Learner}
\label{sec:unknown_prop}

In observational settings, where either the treatment policy is unknown or choice of treatment is endogenous, we can no longer directly run Algorithm~\ref{algo:abs}. One might think to salvage the algorithm by replacing the propensity $p_o(X)$ in the definition of $V_n$ with an ML estimate $\wh{p}(X)$. However, unless $\wh{p}$ converges to $p_o$ in probability at fast, parametric rates\footnote{In particular, one would need $\|p_o - \wh{p}\|_{P_X, 2} = O_\P(n^{-1/2})$}, Algorithm~\ref{algo:abs} will exhibit sub-optimal regret. Instead, one must leverage more sophisticated methods to still enable low-regret learning.

Taking inspiration from the literature on semiparametric estimation~\citep{chernozhukov2018double, bang2005doubly, foster2023orthogonal, athey2021policy}, we introduce a \textit{doubly-robust} analogue of Algorithm~\ref{algo:abs} that furnishes fast regret rates even when the propensity is unknown. Our algorithm uses ML estimates $\wh{g}, \wh{p}$ of the regression $g_o(d, x)$ and propensity $p_o(x)$ to ``de-bias'' observed outcomes $Y$ into more robust pseudo-outcomes. These nuisance estimates are assumed to be independent of the sample. In more detail, we define 

\begin{equation*}
    \wh{\varphi}(x, d, y) := \wh{g}(d, x) + \left(\frac{d \cdot D}{\wh{p}(x)} + \frac{(1-d)\cdot (1 - D)}{1 - \wh{p}(x)}\right)(y - \wh{g}(d, x))
\end{equation*}
This quantity is known as pseudo-outcomes because when $\wh{g} = g_o$ and $\wh{p} = p_o$, one can check that $\E[\wh{\varphi}(X, d, Y) \mid X] = \E[Y(d) \mid X]$. With pseudo-outcomes, we can then define the doubly-robust abstaining welfare/value on the sample, $V_{n, \dr}^{(p)}$, via
{\small
\begin{align*}
&V_{n, \dr}^{(p)} := \E_n \Big[\mathbbm{1}_{\pi(X)\neq *}\Big\{\pi(X)\wh{\varphi}(X, 1, Y) \\
&\;\; + (1 - \pi(X))\wh{\varphi}(X, 0, Y)\Big\} +\mathbbm{1}_{\pi(X)=*}\Big(\tfrac{\hat{\varphi}(X,1,Y)+\hat{\varphi}(X,0,Y)}{2}+p\Big)\Big]
\end{align*}
}
\noindent When $p = 0$, we arrive at the non-abstaining sample welfare $V_{n,\dr} := V^{(p)}_{n,\dr}$. To obtain a doubly-robust analogue of Algorithm~\ref{algo:abs}, use and expanded $\alpha$ \footnote{We may simply increase the constant in Step \ref{step_abs_dr:almost_optimal} of Algorithm~\ref{algo:abs_dr}; when the nuisance product error is $o_p(n^{-1/2})$, this yields the same rate as Theorem~\ref{thm:abstention-rate}.}  and simply replace every occurrence of $V_n$,  $V^{(p)}_n$ by $V_{n, \dr}$ and $V^{(p)}_{n, \dr}$, respectively. See Algorithm \ref{algo:abs_dr} in Appendix \ref{appendix:dr}.

We now state the main theorem of this subsection.

\begin{theorem}
\label{thm:dr-abstention}
For any $p > 0$, suppose Algorithm \ref{algo:abs} is run with $V^{(p)}_{n,\dr}, V_{n, \dr}$ in place of $V^{(p)}_n, V_n$ and expanded $\alpha$. Further, suppose Assumptions~\ref{assump:data} and \ref{assump:vc} hold and that the learner is given nuisance estimates $\wh{g}, \wh{p}$ that are independent of the sample. Then, for any $\delta\in(0,1)$, with probability at least $1-\delta$,
\[
\Reg_n^{(p)}(\pitil) 
\lesssim
\frac{d \log\!\frac{n}{d}+\log\!\frac{1}{\delta}}{p\,n\,\kappa^2}
+  \frac{\err_{\dr}^2}{p \, \kappa^2},
\]
where $\err_{\dr}$ is a known upper bound on the product error given by:
{\small
\[
 \E \l[(\hat{p}(X)\! -\!p_o(X))^2\sum_{d=0}^1(\hat{g}(d,X) \! -\!g_o(d,X))^2 \r]^{1/2}.
\]
}
\end{theorem}
\noindent See Appendix~\cref{appendix:dr} for the full proof and supporting lemmas.

The above theorem yields a bound on regret that has two terms --- one quickly decaying term followed by another, random term whose decay is non-obvious. We now discuss conditions under which this second term is negligible. Note that one can obtain (via Cauchy-Schwarz) the following upper bound on the product error in terms of individual error rates in nuisance functions,
\[
 \big\|\hat p-p\big\|_{P_X, 4} \sum_{d \in \{0,1 \}}\big\|\ \hat g(d,\cdot)-g_o(d,\cdot)\big\|_{P_X,4}
\]
Thus, for the second term in Theorem~\ref{thm:dr-abstention} to be negligible, we need $\err_{\dr} \leq c_{\delta}/\sqrt{n}$ with high probability, which will occur if {\small $\max\{\|\hat p-p\|_{P_X,4},\|\hat g(0,\cdot)-g_o(0,\cdot)\|_{P_X,4},\|\hat g(1,\cdot)-g_o(1,\cdot)\|_{P_X,4}\}$} $\lesssim n^{-1/4}$ with probability at least $1 - \delta$. In particular, this can be accomplished under a variety of learnability assumptions (finite VC-dimension, learnability by trees, etc.) 
\newline
\begin{remark}
In the statement of Theorem~\ref{thm:dr-abstention}, we assume that the nuisance estimates $\wh{g}$ and $\wh{p}$ are independent of the sample. In practice, this could be accomplished by actually performing a three-fold split of the data, reserving the third fold for nuisance estimation and using the first two folds as outlined in Algorithm~\ref{algo:abs}. Another more sophisticated approach would be to use $\calD_2$ to produce nuisance estimates $\wh{g}_1, \wh{p}_1$ to use in defining $V_{n, \dr}$ (which is defined in terms of samples in $\calD_1$), and analogously use $\calD_1$ to build estimates $\wh{g}_2, \wh{p}_2$ used in definition $V^{(p)}_{n, \dr}$.
\end{remark}

%% file: sections/applications.tex
\section{Applications }\label{sec:applications}

In the previous section, we established a formal framework for performing policy learning with abstention. We now explore connections between our results in Section~\ref{sec:abstention} (in particular, Algorithm~\ref{algo:abs}) and other aspects of policy learning. This includes developing (non-abstaining) algorithms for policy learning without standard margin assumptions, safely improving policies relative some baseline treatment strategy, and relating abstaining value/welfare $V^{(p)}$ to policy learning in the presence of distribution shift.

\subsection{Fast Learning Rates Without Standard Margin Assumptions}
 
In policy learning, fast regret rates are sometimes achievable when the CATE is deterministically bounded away from zero, as outlined in the margin condition of Equation~\eqref{eq:margin}. In particular, prior works have shown that $O(n^{-1})$ rates are possible in the \emph{realizable} setting, where the Bayes-optimal policy $\pi^B$ lies in the class $\Pi$ \citep{kitagawa2018should, Luedtke2020}. However, it is generally impossible to know if a policy class $\Pi$ (say, all depth 3 decision trees) contains the optimal treatment policy in advance. Thus, one should aim to obtain fast regret rates under in \textit{agnostic} settings, where $\Pi$ may not contain the optimal treatment strategy.

In the classification literature, recent works achieve fast learning rates in agnostic settings~\citep{ben2014sample, bousquet2021fast}. These works eschew the realizability assumption, instead imposing a \textit{finite combinatorial diameter} on the policy class. In words, a class of policies $\Pi$ which bounds the maximal number of points of disagreement between any two policies in the class. More formally, the combinatorial diameter is defined by

\begin{small}
   \begin{equation*}
D \;\coloneqq\; \max_{\pi_1,\pi_2\in\Pi} \sum_{x\in\cX} \indic{\pi_1(x)\neq \pi_2(x)}.
\end{equation*} 
\end{small}

Note that $D$ may be infinite even for classes with finite VC dimension. 

We prove an analogous result for policy learning (Algorithm \ref{thm:finite-D}) under the assumption of a finite combinatorial diameter. Further, we consider settings where the policy class does not have a finite combinatorial diameter but instead we have access to a CATE estimation oracle (Algorithm \ref{thm:reg-oracle}). The proofs are deferred to Appendix \ref{appendix:applications}. We start by stating our main theorem for agnostic policy learning under an assumption of finite combinatorial dimension.
{\textcolor{black}{
\begin{algorithm}[t]
\caption{Policy Learning under Margin Assumption}
\label{algo:fast-rate-wrapper}
\begin{algorithmic}[1]
\STATE \textbf{Input:} Samples $\{(X_i,D_i,Y_i)\}_{i=1}^n$, policy class $\Pi$, overlap $\kappa$, margin $h>0$, confidence $\delta$, \textsf{Mode}$\in\{\textsf{FiniteD},\textsf{CATE-Oracle}\}$.
\STATE \textbf{Split:} Partition the sample into three equal parts: $\mathcal{D}_1,\mathcal{D}_2,\mathcal{D}_3$ (sizes $n/3$).
\STATE \textbf{Abstention stage:} On $\mathcal{D}_1\cup\mathcal{D}_2$, run \cref{algo:abs} with bonus $p=h/2$ to obtain a policy $\pitil$ that may abstain. Let $\cX_{rem}\subseteq\cX$ be the covariate set where $\pitil$ abstains.
\IF{\textsf{Mode}=\textsf{FiniteD}}
  \STATE \textbf{Refine on $\cX_{rem}$ (finite $D$):} On $\mathcal{D}_3$, run EWM over all the policies on  $\cX_{rem}$.
  \STATE $\phi \leftarrow \argmax \E_n[ v(\pi, X) | X \in \cX_{rem}]$.
\ELSIF{\textsf{Mode}=\textsf{CATE-Oracle}}
  \STATE \textbf{Refine on $\cX_{rem}$ (CATE oracle):} On $\mathcal{D}_3$, estimate $\tauh$ on $\cX_{rem}$ and set $\phi(x)\leftarrow\indic{\tauh(x)>0}$ for $x\in\cX_{rem}$.
\ENDIF
\STATE Define: $
\pi_{\text{final}}(x)=
\begin{cases}
\pitil(x), & x\notin \cX_{rem},\\
\phi(x), & x\in \cX_{rem}
\end{cases}$
\STATE \textbf{Output:} $\pi_{\text{final}}$.
\end{algorithmic}
\end{algorithm}
}
 }
\begin{theorem}
\label{thm:finite-D}
Assume the margin condition $\P(\abs{\tau_o(X)}\geq h)=1$ and that Assumptions~\ref{assump:data} and \ref{assump:vc} hold. Further, assume $\Pi$ has finite combinatorial diameter $D$. Then the output $\pi_{\text{final}}$ of Algorithm~\ref{algo:fast-rate-wrapper} satisfies
\[
\Reg_n(\pi_{\text{final}}) := V(\pi^*) - V(\pi_{\text{final}})
\lesssim
\frac{ D \;+\; d\log\!\frac{n}{d} \;+\; \log\!\frac{1}{\delta} }{ \kappa^2 \, h\, n }
\]
with probability at least $1 - \delta$.
\end{theorem}
\noindent Proof details are provided in Appendix~\cref{appendix:applications}.

Next, suppose we have an oracle for estimator $\wh{\tau}_o$ that satisfies for any $(X,D,Y) \sim \prob$ the following convergence rate with probability greater than $1-\delta$,
\[
\norm{\tauh - \tau}{\prob, 2} \leq c_\delta n^{-\beta}.
\]

\begin{theorem}
\label{thm:reg-oracle}
Assume the margin condition $\P(\abs{\tau_o(X)}\geq h)=1$ and that Assumptions~\ref{assump:data} and \ref{assump:vc} hold. Further, suppose $\tauh$ satisfies the regression-oracle condition with exponent $\beta>0$ and let $\pi^B$ bet the Bayes optimal policy. Then the output $\pi_{\text{final}}$ of Algorithm~\ref{algo:fast-rate-wrapper} satisfies 
\[
V(\pi^*) - V(\pi_{\text{final}})
\lesssim & \; \mathrm{Reg}_1 + \mathrm{Reg}_2 + \mathrm{Reg}_3\]
where,
\begin{small}
\[
\mathrm{Reg}_1  &\lesssim n^{-1} \; \frac{ d\log\!\frac{n}{d} \;+\; \log\!\frac{1}{\delta} }{ \kappa^2 \, h }, 
\\
\mathrm{Reg}_2 &\lesssim  n^{-2 \beta} \frac{c_\delta \l( V(\pi^B(X)) - V(\pi^*(X)\r)^{1-2 \beta}}{h^{2-2\beta}}, \\ 
\mathrm{Reg}_3 & \lesssim n^{- \frac{1 }{2}- \beta} \;  \frac{\l(d\log\!\frac{n}{d} + \log\!\frac{1}{\delta} \r)^{\frac{1}{2} - \beta}}{\kappa^{1- 2\beta} h^{2-2\beta} }.
\]
\end{small}
with probability at least $1 - \delta$.  
\end{theorem}
\noindent Proof details are provided in Appendix~\cref{appendix:applications}.

Theorem \ref{thm:reg-oracle} unifies the regression-rate guarantee \citep{Luedtke2020} and fast rate in the realizable case \citep{kitagawa2018should}. In particular, when $\beta = 1/2$, one recovers $n^{-1}$ regret. More generally, if $\beta < 1/2$ but the policy class $\Pi$ is expressive enough that $V(\pi^B(X)) - V(\pi^*(X))$ is small (e.g., $\leq c/n^{-\alpha}$ for some $\alpha>0$), the resulting bound improves upon a direct plug-in classifier ($\pi(X) = \indic{\hat \tau(X) > 0}$). Thus, either a strong CATE oracle or an ``almost'' realizable policy class is sufficient to surpass the $1/\sqrt{n}$ regret.

\begin{figure*}[t]
  \centering
  \includegraphics[width=1\linewidth]{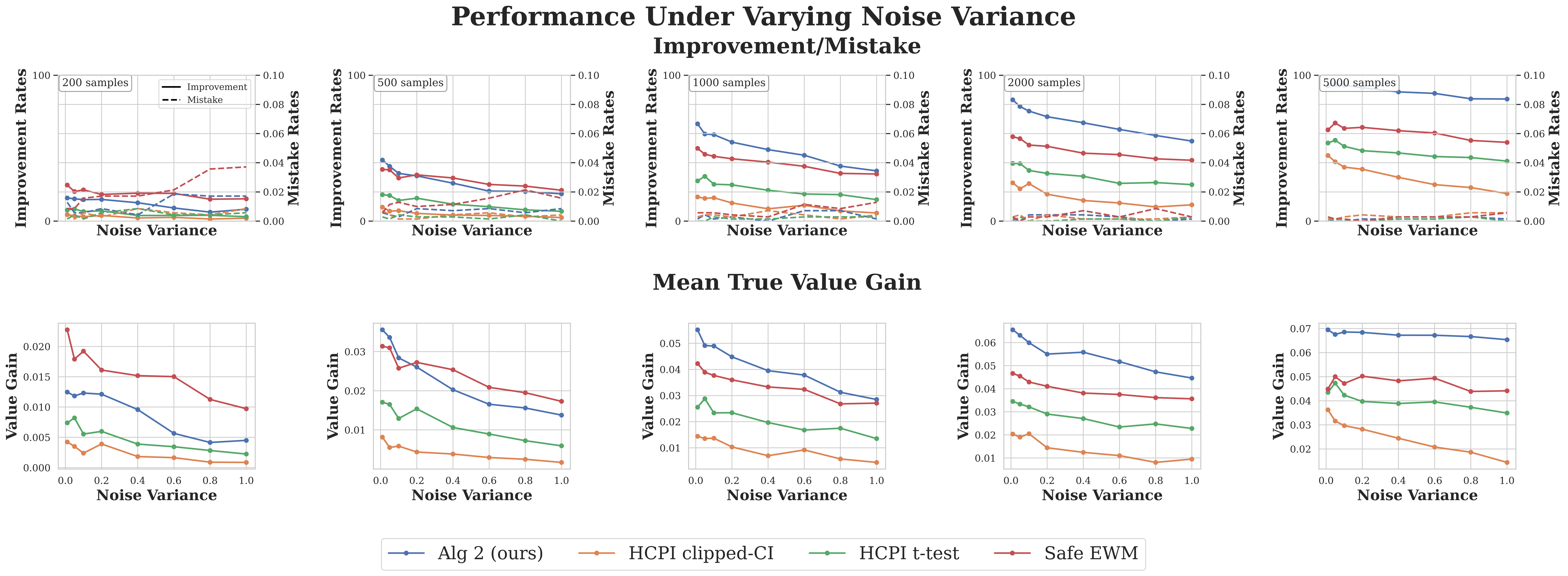}
  \caption{Safe Policy Improvement — performance across varying noise variance. We compare Algorithm~2 (ours), HCPI (two variants), and Safe EWM.}
  \label{fig:varying_noise}
\end{figure*}

\subsection{Safe Policy Improvement}
\label{sec:safe_policy}
Safe policy improvement (SPI) addresses the deployment problem under uncertainty: one should adopt a new policy only if it can be certified to outperform a baseline policy $\omega$ with high probability.  Following recent work on safe policy improvement \cite{thomas2015high, cho2025cspi}, we use sample splitting (into $\calD_{\text{train}}$ and $\calD_{\text{test}}$) to separate the tasks of \emph{candidate policy selection} and \emph{safety testing}. This splitting enables one-sided lower confidence bounds (LCBs) on the improvement over the baseline: $V(\hat\pi)-V(\omega)$. We treat Algorithm \ref{algo:abs} as a black box to propose a \(\{0,1,*\}\)-valued policy $\pitil$ learned on $\mathcal{D}_{train}$. To obtain a deployable binary policy, we impute abstentions with the baseline policy to obtain a candidate policy Equation \eqref{eq:replace}. 


On $D_{\text{test}}$, we estimate $V(\hat\pi)-V(\omega)$ via IPW/DR and compute a one-sided LCB at level $1-\delta$; we generate a sequence of candidate policies by running the abstention learner over a grid of bonuses, $\mathcal{P}=\{p_1<\cdots<p_k\}$, which trades off overlap (more abstention) against improvement (more overrides), and apply a Bonferroni adjustment with $z_{1-\delta/k}=\Phi^{-1}(1-\delta/k)$. Algorithm~\ref{algo:safe_policy_abstention} then: (i) selects the grid, (ii) learns abstaining policies on $D_{\text{train}}$, (iii) replaces abstentions with $\omega$, and (iv) tests each on $D_{\text{test}}$, returning the first with $\mathrm{LCB}>0$.


\begin{small}
\begin{algorithm}[tb]
\caption{Safe Policy Learning with Abstention}\label{algo:safe_policy_abstention}
\begin{algorithmic}[1]
  \STATE Select a sequence of abstention bonuses ${\cal{P}} = \{ p_1, p_2, \ldots, p_k \}$, with $p_1 < \ldots < p_k $.
\STATE Split the dataset into $D_{train}$ and $D_{test}$ of size $n_{train}$ and $n_{test}$ respectively.
\FOR{ $p \in \cal{P}$}
    \STATE Run  \cref{algo:abs}  for $V^{(p)}(\pi)$ on $D_{train}$ to obtain an abstaining  policy $\pitil$
    \STATE Replace abstention with the baseline policy 
\begin{equation}\label{eq:replace}
   \hat\pi(x) \;=\;
\begin{cases}
\omega(x), & \tilde\pi(x)=*,\\
\tilde\pi(x), & \text{otherwise}.
\end{cases} 
\end{equation}
\vspace{-1.5em}
    \STATE \texttt{Hypothesis test: }
    \STATE Estimate $V_n(\pihat) - V_n(\omega)$ on $D_{test}$. 
    \STATE Compute the one-sided $(1-\delta)$ lower confidence bound:
  $\mathrm{LCB}=V_n(\pihat) - V_n(\omega) - z_{1-\frac \delta k}\,\frac{\hat\sigma_k}{\sqrt{n_{\text{test}}}}.$
    \STATE \textbf{if} $\mathrm{LCB}>0$ \textbf{then} \textbf{return} $\hat\pi$ \textbf{else} continue.

\ENDFOR
\end{algorithmic}
\end{algorithm}
\end{small}
\begin{figure*}[t]
  \centering
  \includegraphics[width=1\linewidth]{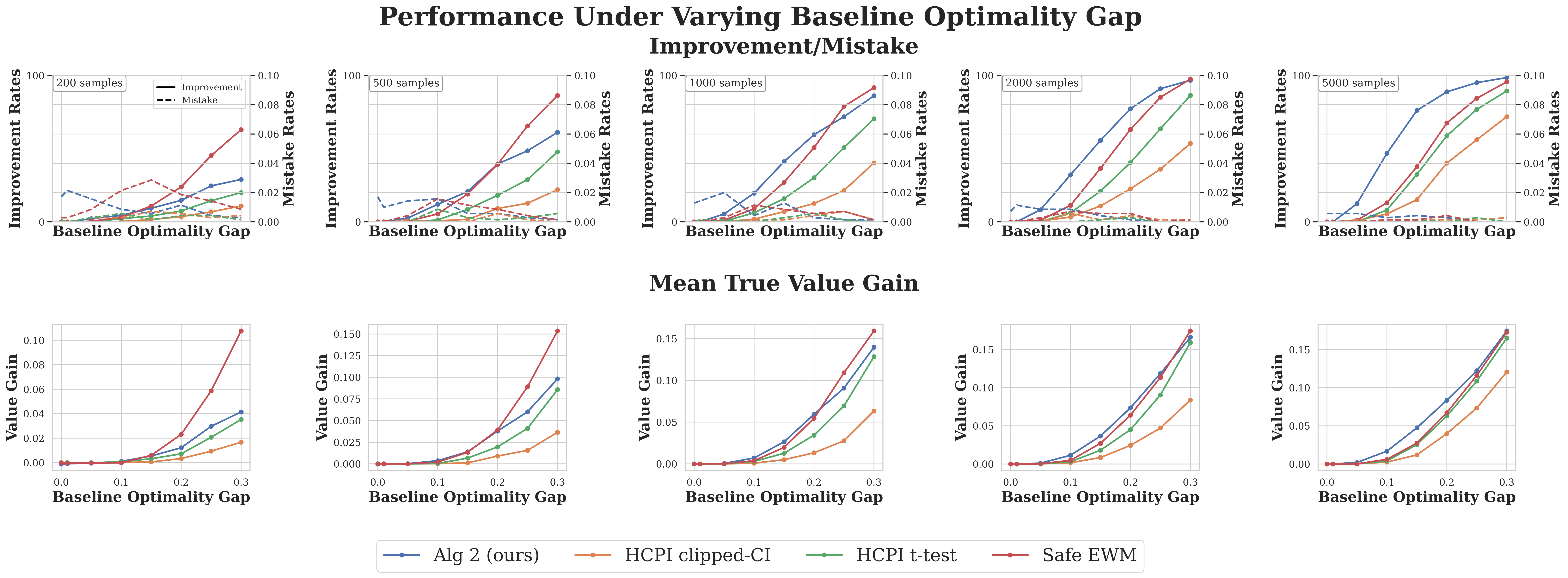}
  \caption{Safe Policy Improvement — performance across varying gap between baseline and optimal policy value. We compare Algorithm \ref{algo:safe_policy_abstention} (ours), HCPI (two variants), and Safe EWM.}
  \label{fig:varying_base_opti_gap}
\end{figure*}

\paragraph{Safe Policy Improvement Experiments}

We compare Algorithm \ref{algo:safe_policy_abstention} with an Empirical Welfare Maximizer (EWM) and prior work on synthetic data. In particular, we consider the two variants of High-Confidence Policy Improvement (HCPI) described in \cite{thomas2015high}: one based on finite-sample confidence intervals and one based on a $t$-test statistic. For EWM, we choose the empirical welfare maximizer as the candidate policy and test it against the baseline in exactly the same way as in Algorithm \ref{algo:safe_policy_abstention} (with $k=1$).

The design comprises two variants: (1) we use a baseline policy at a fixed a and vary the noise variance $(0.01\text{--}1.0)$ (see \cref{fig:varying_noise}); and (2) we vary the baseline–optimal value gap at a fixed noise level (see \cref{fig:varying_base_opti_gap}). Each parameter setting is replicated 100 times; policy value is computed by IPW (with known propensities). For Algo \ref{algo:safe_policy_abstention}, we use a grid of abstention bonuses \(\mathcal{P} = \{0, 0.01, 0.05, 0.10, 0.20\}\). We report value gain, mistake rate (probability of returning a policy worse than the baseline), and improvement rate (probability of returning a policy strictly better than the baseline). Full configuration and implementation details are deferred to Appendix \ref{appendix:experiments}.

We evaluate all methods at significance level \(\delta = 0.05\). Across more than \(2{,}000\)  repetitions, our Safe Policy Learner controls the Type-I error at or below \(0.05\). In the very small–sample regime ($n\!\le\!500$), EWM performs better than the alternatives, but once $n\!\ge\!1000$ our method dominates: \emph{(i)} it achieves the highest improvement rate and the largest mean value gain among accepted policies across baseline–optimal gaps (\cref{fig:varying_base_opti_gap}), and \emph{(ii)} it maintains both higher power and competitive mistake rates as noise increases (\cref{fig:varying_noise}). Overall, for moderate-to-large samples, Algorithm~\ref{algo:safe_policy_abstention} delivers the best safety–power tradeoff, improving more often and by more, while meeting the $\delta$-level error control.

\subsection{Robustness to Distribution Shifts}
\label{subsec:distribution-robust}
The abstention value $V^{(p)}$ has a natural interpretation as protection against outcome distribution shift. Specifically, consider a setting where the true data distribution of $(Y(0),Y(1))\mid X$ is different from our training data. This would happen when our observational data is outdated and cannot reflect the true effect of the current treatment of interest.

While deterministic policies are optimal without such an outcome distribution shift, they can be problematic otherwise.  To mathematically formulate this intuition, we assume the true potential outcome distribution lies in some $W_1$-ball of the distribution that generates our observations, \emph{i.e.}

\begin{small}
\begin{equation}\notag
\prob_{\mathtt{test}} \in \mathcal{P}_{\alpha}(\prob_{\mathtt{train}})
\;\coloneqq\;
\Big\{\prob:\; W_1(\prob,\prob_{\mathtt{train}})\le\alpha\Big\}.
\end{equation}
\end{small}

The robust objective is then to maximize a policy’s worst–case value over $\mathcal{P}_{\alpha}(\prob_{\mathtt{train}})$.
Specifically, consider policies $\pi$ taking values in $\{0,1,\tfrac12\}$, where $\tfrac12$ denotes a uniform random choice between the two treatments. In this model, the robust objective coincides with maximizing the abstention value $V^{(p)}$ for a specific bonus level, as formalized below.
 
\begin{proposition}
\label{prop:dist-shift} (i). For any (possibly random) policy $\pi: \mathcal{X}\mapsto[0,1]$, there exists $\hat{\pi}: \mathcal{X}\mapsto\{0,\frac{1}{2},1\}$ such that $\pi\in\{0,1\}\Rightarrow\pi=\hat{\pi}$ and $\min_{\prob\in\mathcal{P}_{\alpha}} V(\hat{\pi}) > \min_{\prob\in\mathcal{P}_{\alpha}} V(\pi)$.\\
(ii). Let $\pi:\mathcal{X}\mapsto\{0,\frac{1}{2},1\}$ and $\tilde{\pi}:\mathcal{X}\mapsto\{0,1,*\}$ be obtained by replacing all $\pi(X)=\frac{1}{2}$ with $*$, then we have
\begin{equation}
\min_{\prob\in\mathcal{P}_{\alpha}} V(\pi)
\;=\;
V^{(\alpha/2)}(\pi)\;-\;\alpha.
\end{equation}
\end{proposition}

Distributional uncertainty acts like an outcome–level penalty of size $\alpha$ when the policy commits to an arm, and a half–penalty under randomization; hence, maximizing the robust objective is equivalent to maximizing $V^{(p)}$ with $p=\alpha/2$.


%% file: sections/conclusion.tex
\section{Conclusion}

In this paper, we introduced a framework for policy learning with abstention. Building off of the classification with abstention literature, we formulated a value/welfare that incentivizes abstention by offering a small, additive reward on top of the value of a random guess in regions of uncertainty. We described algorithms that obtain fast abstaining regret rates in both known and unknown propensity settings, and showed how to apply our algorithm to various disparate policy learning problems. There are still many other interesting open directions related to policy learning with abstention. First, our framework only handles binary treatments. There are many settings (such as in medicine) where designing policies to optimize a continuous treatment (or dosage) may be more appropriate. Second, implementing Algorithm \ref{algo:abs} efficiently for popular policy classes such as decision trees is also an interesting open direction.  Likewise, we evaluate the performance of algorithms through regret, but one could also consider loss-based modes of policy evaluation as well. 

%% file: sections/appendix-main.tex
\section{Missing Proofs from \cref{sec:abstention}}\label{appendix:main}
\paragraph{Additional notation.}
Let $d$ denote the VC dimension of $\Pi$ and set
\[
\alpha \;\coloneqq\; \sqrt{\frac{\,d \log\!\frac{n}{d} \;+\; \log\!\frac{1}{\delta}\,}{n}}\,.
\]
We use $\const$ to denote absolute constants. 

Recall that $\pihat$ is the policy selected by the first EWM step (Step~\ref{step_abs:first_ewm} in \cref{algo:abs}); $\Pihat$ is the set of ``almost'' optimal policies (Step~\ref{step_abs:almost_optimal}); $\Pitil$ is the set of abstaining policies constructed from $\Pihat$ (Step~\ref{step_abs:abs_projection}); elements of $\Pitil$ are denoted by $\pi'$, and $\pitil$ is the final abstaining policy returned by \cref{algo:abs} (Step~\ref{step_abs:second_ewm}). The class $\Pitil$ is formed by projecting disagreements of policies in $\Pihat$ with the fixed reference $\pihat\in\Pi$. Notationally, we write $\phi$ for a generic element of $\Pitil$ and $\phi^*$ for a maximizer of $V^{(p)}$. For brevity, we also use $\pi'$ to denote the abstaining policy obtained from a given $\pi$ via its disagreement with $\pihat$; e.g., $\pi'_1,\pi'_2$ are the abstaining policies constructed from $\pi_1,\pi_2$, respectively.

For $f:\cX\to\mathbb{R}$,
\[
\|f\|_{P_X,q}\;\coloneqq\;\Big(\E_{P_X}\big|f(X)\big|^q\Big)^{1/q}.
\]
For $h:\cX\times\{0,1\}\times[0,1]\to\mathbb{R}$,
\[
\|h\|_{P_Z,q}\;\coloneqq\;\Big(\E_{P_Z}\big|h(Z)\big|^q\Big)^{1/q},\qquad Z=(X,D,Y).
\]
On samples $Z_{1:n}$ the empirical $L^2$ norm is denoted as:
\[
\|h\|_{L^2(Z_{1:n})}\;\coloneqq\;\Big(\tfrac{1}{n}\sum_{i=1}^n h(Z_i)^2\Big)^{1/2},
\]
abbreviated $\|h\|_n$ when clear. 
Finally, note that for  for binary policies,
\[
\E\big[\indic{\pi_1(X)\neq \pi_2(X)}\big]\;=\;\|\pi_1-\pi_2\|_{P_X,1}.
\]

\paragraph{VC subgraph dimension.} Let $\cF$ be a class of real-valued functions on $\cX$. The \emph{VC subgraph dimension}  of $\cF$ is the VC dimension of the family of subgraphs
\[
\mathsf{Subgraph}(\cF)\;\coloneqq\;\bigl\{\,\{(x,t)\in\cX\times\mathbb{R}:\ t<f(x)\}: \ f\in\cF\,\bigr\}.
\]
Equivalently, for any finite set $\{(x_i,t_i)\}_{i=1}^n\subset\cX\times\mathbb{R}$, we say $\mathsf{Subgraph}(\cF)$ \emph{shatters} this set if for every labeling $b\in\{0,1\}^n$ there exists $f\in\cF$ such that
\[
\mathbf{1}\{t_i<f(x_i)\}=b_i,\quad i=1,\dots,n.
\]
The VC subgraph dimension $\mathrm{VC}_{\text{sub}}(\cF)$ is the largest $n$ for which some $n$-point set in $\cX\times\mathbb{R}$ is shattered (or $+\infty$ if no finite maximum exists). 
For $\{0,1\}$–valued classes, $\mathrm{VC}_{\text{sub}}(\cF)$ coincides with the usual VC dimension.

First, we restate that a standard result regarding the VC subgraph dimension (see Lemma A.1 of \citet{kitagawa2018should}):
\begin{lemma}\label{lem:vc-sub}
Let $\Pi$ be a class of binary functions with VC dimension $d$, and define the real-valued class
\[
\cF_\Pi \;=\; \Big\{\, z \mapsto \pi(x)\,\gamma_1(z) \;+\; \big(1-\pi(x)\big)\,\gamma_2(z) \ :\ \pi\in\Pi \,\Big\},
\]
where $\gamma_1,\gamma_2$ are fixed real-valued functions and $x$ denotes the covariate component of $z$. Then the VC subgraph dimension of $\cF_\Pi$ is at most $d$.
\end{lemma}

\subsection{Concentration Inequalities}
This section uses standard techniques from Empirical Process Theory, in particular we use localized Rademacher complexity to obtain Bernstein-type concentration for the IPW estimator.

\begin{lemma}\label{lem:main_conc}
Let $\Pi$ be a class of binary functions with VC dimension $d$, and define
\[
\cF_\Pi \;=\; \Big\{\, z \mapsto \pi(x)\,\gamma_1(z) + \big(1-\pi(x)\big)\,\gamma_2(z) \;:\; \pi\in\Pi \,\Big\},
\]
where $\gamma_1,\gamma_2$ are a.s.\ bounded with $|\gamma_1(Z)|,|\gamma_2(Z)|\le B$. Fix $f_{\pi^*}\in\cF_\Pi$. For any $f_\pi\in\cF_\Pi$ constructed from $\pi$ as above, with probability at least $1-\delta$,
\[
\big|\;\E[ f_{\pi^*}(Z)-f_\pi(Z) ] \;-\; \E_n[ f_{\pi^*}(Z)-f_{\pi}(Z) ]\;\big|
\;\le\;
\const  \Big(\alpha\sqrt{B \; \E|f_{\pi^*}-f_\pi|}+ B \alpha^2\Big),
\]
 where $\alpha \;\coloneqq\; \sqrt{\frac{\,d \log\!\frac{n}{d} \;+\; \log\!\frac{1}{\delta}\,}{n}}$.
\end{lemma}

\begin{proof}
We begin by constructing the star convex function class
    \[\cG = \l\{ \beta \; \frac{f_\pist(z) - f_\pi(z)}{2 \;B}: f_\pi \in \cF_\Pi \;, \; \beta \in [0,1] \r\}.\]
Next, we define the critical radius as 
    \[
    r(n,\delta) = \inf \l\{ r \geq 0 : \prob \l( \sup_{g \in \cG,  \E g^2 \leq r^2 } \abs{(\E - \E_n)g} \leq c r^2 \r)  \geq 1- \delta\r\}
    \]
where $c$ is an absolute constant that will be chosen later. In particular, for such a choice of $r(n,\delta)$, with probability at least $1-\delta$ we have
\begin{align}
    \l| (\E - \E_n)g \r| \le c\l(r(n,\delta) \sqrt{\E g^2} + r(n,\delta)^2\r). \label{eq:critical_radius}
\end{align}

The inequality \cref{eq:critical_radius} holds trivially for any $g$ satisfying $\E g^2 < r(n,\delta)^2$. Now, if $\E g^2 \ge r(n,\delta)^2$, define
\[
g' \coloneqq \frac{r(n,\delta)~ g}{\sqrt{\E g^2}},
\]
so that $\E g'^2 < r(n,\delta)^2$. Applying the inequality to $g'$ and then scaling back, we obtain
\[
\l| (\E - \E_n)g \r|  \le c ~ r(n,\delta) \sqrt{\E g^2},
\]
This gives us the bound in \cref{eq:critical_radius}.
   
Next, observe that since $\abs{g(Z)} \in [0,1]$, it follows that
\[
\E g^2 \le \E\l| g \r|.
\]

By a standard symmetrization argument, we obtain the inequality
\begin{align}
      \E\l[\sup_{\E g^2 \leq r^2 } \abs{(\E - \E_n)g} \r]  \leq 2 \E[ \cR( Z_{1:n}, \cG) ], \label{eq:symmetrization}  
\end{align}
where the Empirical Rademacher complexity is defined by
\[
\cR( Z_{1:n}, \cG) \coloneqq \E_{\epsilon_{1:n}}\l[\sup_{\E g^2 \leq r^2 } \abs{\frac{1}{n} \sum_{i=1}^n \epsilon_i g(Z_i)} \Big| Z_{1:n}\r].
\]

We now introduce 
\[
M = \sup_{\substack{f, g \in \cG \\ \E f^2, \E g^2 \leq r^2} } \norm{f -g }{L^2(Z_{1:n})},
\]
with 
\[
\norm{f -g }{L^2(Z_{1:n})}^2 \coloneqq \frac{\sum_{i=1}^n (f(Z_i) -g(Z_i))^2}{n}.
\]
The following version of Dudley's inequality will be used:
\begin{align}
       \cR(Z_{1:n}, \cG) \leq 4 \gamma + \frac{12}{\sqrt{n}} \int_\gamma^M \sqrt{\log{N\l(\varepsilon, \cG,\norm{\cdot  }{L^2(Z_{1:n})}\r) }} d \varepsilon 
\quad \quad \forall \gamma >0. \label{eq:dudley}
\end{align}

To bound the covering numbers, we use the fact that the covering numbers of a star-shaped hull $\cG$ of a class $\cF$ (of VC dimension $d$) satisfy \citep{bartlett2006empirical}:
\[
N\l(\varepsilon, \cG,\norm{\cdot  }{L^2(Z_{1:n})}\r) 
\leq
N\l(\varepsilon, \cF,\norm{\cdot  }{L^2(Z_{1:n})}\r) \l( 1 + \l\lceil \frac{2}{\varepsilon} \r\rceil \r)
\leq  e(d+1)\l(\frac{2e}{\varepsilon}^2\r)^d \l( 1 + \l\lceil \frac{2}{\varepsilon} \r\rceil \r)
\leq  e(d+1)\l(\frac{2e}{\varepsilon}^2\r)^{d+1}.
\]
    The second-to-last inequality follows from standard bounds on the covering numbers of VC classes. Substituting $\gamma = d/n$ into \cref{eq:dudley}, we get

    \begin{align}
            \cR( Z_{1:n}, \cG) &\leq \frac{4d}{n} + \frac{12}{\sqrt{n}} \int_{d/n}^M \sqrt{(d+1)\log{\frac{2e}{\varepsilon}} + \log{(d+1)}}  ~ d \varepsilon  \nn \\
     &\leq \frac{4d}{n} + \frac{12}{\sqrt{n}} \cdot \sqrt{(d+1)\log{\frac{2e n}{d}} + \log{(d+1)}} ~  \int_{d/n}^M   ~ d \varepsilon  \nn \\
     &\leq \frac{\const d}{n} + \frac{\const M}{\sqrt{n}} \sqrt{d\log{\frac{n}{d}}} .\label{eq:emp_rademacher_bound} 
    \end{align}
    
  We also have 
\begin{align}
\E M &= \E \sqrt{\frac{\sup_{ \E g^2 \leq r^2 } \sum_i g^2(Z_i) }{n}} \nn \\ 
&\le \sqrt{\frac{\E \sup_{ \E g^2 \leq r^2 } \sum_i g^2(Z_i) }{n}}  \tag{Jensen's inequality} \\ 
&\le \sqrt{\frac{\E \sup_{ \E g^2 \leq r^2 } \sum_i \l( g^2(Z_i) - \E g^2(Z_i) \r) }{n} + r^2} \nn \\ 
&\le \sqrt{ \frac{2 ~\E \sup_{ \E g^2 \leq r^2 } \sum_i \epsilon_i g^2(Z_i)  }{n} + r^2} \tag{via Symmetrization} \nn \\
&\le \sqrt{\frac{2 ~ \E \sup_{ \E g^2 \leq r^2 } \sum_i \epsilon_i g(Z_i)  }{n} + r^2} \tag{via Contraction Lemma \citep{shalev2014understanding}} \nn \\
&\le \sqrt{\frac{2 ~\E \sup_{ \E g^2 \leq r^2 } \sum_i \epsilon_i g(Z_i)  }{n} } + r \nn \\
&=\sqrt{\frac{2~ \E ~ \cR(Z_{1:n}, \cG)}{n}}+ r. \nn
\end{align}

    Taking expectation in \cref{eq:emp_rademacher_bound} we have 
    \[
    \E[ \cR( Z_{1:n}, \cG) ] &\le \frac{\const d}{n} + \frac{\const \E M}{\sqrt{n}} \sqrt{d\log{\frac{n}{d}}} \\
    &\leq \const r\sqrt{\frac{d \log{\frac{n}{d}}}{n} } + \const \sqrt{\frac{\E \cR(Z_{1:n}, 
    \cG)}{n} } \sqrt{\frac{d \log{\frac{n}{d}}}{n} }+ \frac{\const d }{n} 
    \]
The inequality above, together with \cref{eq:symmetrization}, implies that
    \[
        \E\l[\sup_{\E g^2 \leq r^2 } \abs{(\E - \E_n)g} \r]  \leq 2 ~ \E[ \cR( Z_{1:n}, \cG) ] \leq \const\l( r\sqrt{\frac{d \log{\frac{n}{d}}}{n} } + \frac{d \log{\frac{n}{d}}}{n} \r)
    \]
Finally, by applying Talagrand's inequality (Theorem 3.27 of \cite{wainwright2019high}), we obtain
\[
\prob\l( \sup_{\E g^2 \leq r^2 } \abs{(\E - \E_n)g} \geq \const\l( r\sqrt{\frac{d \log{\frac{n}{d} + \log{\frac{1}{\delta}}}}{n} } + \frac{d \log{\frac{n}{d} + \log{\frac{1}{\delta}}}}{n}  \r)  \r) \leq \delta.
\]
Now we invoke the Peeling Argument \citep{wainwright2019high}. 
For $k=0,1,\dots,K:=\lceil \log_2(n) \rceil$, define shells
$\cG_k=\{g\in\cG:2^{-(k+1)}< \E g^2 \le 2^{-k}\}$.
Apply the tail bound to each $\cG_k$ with confidence levels
$\delta_k=\delta/(K+1)$ and union bound over $k$. On the resulting
$1-\delta$ event we have, simultaneously for all $g\in\cG$,
\[
|(\E-\E_n)g| \;\le\; \const \Big(\alpha\sqrt{\E g^2}+\alpha^2\Big),
\quad
\alpha:=\sqrt{\frac{d\log\!\frac{en}{d}+\log\!\frac{1}{\delta}}{n}}.
\]

Taking $r(n,\delta)\asymp \alpha$ gives the desired localized bound. Finally, since $g=(f_{\pi^*}-f_\pi)/(2B)$ and $\E g^2=\frac{\E\,(f_{\pi^*}-f_\pi)^2}{(2B)^2}\le \frac{\E|f_{\pi^*}-f_\pi|}{2B}$,
\[
\abs{(\E-\E_n)\big(f_{\pi^*}-f_\pi\big)}\;\le\; \const \Big(\alpha\sqrt{B \; \E|f_{\pi^*}-f_\pi|}+ B \alpha^2\Big),
\]
which is exactly the stated inequality.
\end{proof}

\subsection{Proof of \cref{thm:abstention-rate} }

For $z=(x,y,d)$, recall the \emph{normalized IPW score}
\[
f_{\pi}(z)\;\coloneqq\; \kappa\!\left(\pi(x)\,\frac{y d}{p_o(x)} \;+\; \big(1-\pi(x)\big)\,\frac{y(1-d)}{1-p_o(x)}\right),
\]
so that $f_{\pi}(Z)\in[0,1]$ almost surely (by $Y\in[0,1]$ and $p_o(x)\in[\kappa,1-\kappa]$), and write
$\E[\cdot]$ for population expectation and $\E_n[\cdot]$ for the empirical average over the sample.

\begin{corollary}\label{lem:conc1}
For any policy $\pi$, with probability at least $1-\delta$, we have
\[
\l| V(\pi^*) - V(\pi) - \l( V_n(\pi^*) - V_n(\pi) \r) \r| \leq \frac{\const}{\kappa} \l( \alpha \sqrt{\E\l| f_\pi - f_{\pi^*} \r|} + \alpha^2 \r).
\]
\end{corollary}

\begin{proof}
    We invoke \cref{lem:main_conc} with $\gamma_1(z) = \frac{yd}{p_o(x)}$ and $\gamma_2(z) = \frac{y(1-d)}{1- p_o(x)}$ and setting $B= \frac{1}{\kappa}$ gives us the stated bound. 
\end{proof}

\begin{corollary}\label{lem:conc2}
With probability at least $1-\delta$, for any $\pi_1,\pi_2\in\Pi$ the following  holds for some constant $\const$:
\begin{align}
\Big|\, \E_n\big| f_{\pi_1} - f_{\pi_2} \big| - \E\big| f_{\pi_1} - f_{\pi_2} \big| \,\Big|
&\le \const \Big( \alpha \sqrt{\E\big| f_{\pi_1} - f_{\pi_2} \big|} + \alpha^2 \Big).
\end{align}
\end{corollary}

\begin{proof}
We invoke \Cref{lem:main_conc}. In particular, note that
\[
\l| (f_{\pi_1} - f_{\pi_2})(x,y,d) \r| = \indic{\pi_1(x) \neq \pi_2(x)} \l| \frac{\kappa\ y\, d}{p_o(x)} + \frac{\kappa y\,(1-d)}{1-p_o(x)} \r|.
\]
Moreover, the class
\[
\mathcal{H} = \l\{ x \rightarrow \indic{\pi_1(x) \neq \pi_2(x)} : \pi_1,\pi_2\in\Pi \r\}
\]
has VC dimension is at most $10d$ \cite{bousquet2021fast, vidyasagar2013learning}. We now invoke \cref{lem:main_conc} with $\Pi = \mathcal{H}$ ,  $\gamma_1(z) = \l| \frac{\kappa\ y\, d}{p_o(x)} + \frac{\kappa y\,(1-d)}{1-p_o(x)} \r|$,  $\gamma_2(z) = 0$ and $B=2$.  
\end{proof}

Using \cref{lem:conc1} and \cref{lem:conc2} we have
\begin{corollary}\label{lem:conc1}
For any policy $\pi$, with probability at least $1-\delta$, we have
\[
\l| V(\pi^*) - V(\pi) - \l( V_n(\pi^*) - V_n(\pi) \r) \r| \leq \frac{\const}{\kappa} \l( \alpha \sqrt{\E_n\l| f_\pi - f_{\pi^*} \r|} + \alpha^2 \r).
\]
\end{corollary}

\begin{lemma}\label{lem:conc3}
Let $\Pi$ be a class of binary-valued functions with VC dimension $d$, and fix $\pihat \in\Pi$. Suppose we construct a $\{0,1,* \}$-valued policy class $\Pitil$
\[
 \pi'(x) = \begin{cases}
     \pihat(x) & \text{if} ~~~ \pihat(x) = \pi(x)\\
     * \quad & \text{otherwise} 
         \end{cases} \qquad \Pitil =\{\pi': \forall \pi \in \Pihat \}
\]
Further, suppose $\cD(\Pi) = \sup_{\pi_1, \pi_2 \in \Pi}\norm{\pi_1 - \pi_2}{P_X, 1}$Then, with probability at least $1-\delta$, for every $\pi' \in\widetilde{\Pi}$ we have
\[
\abs{V^{(p)}(\phi^*) - V^{(p)}(\pi') - \l(V_n^{(p)}(\phi^*) - V_n^{(p)}(\pi')\r)} \lesssim  \sqrt{\cD(\pi)} \; \frac{\alpha}{\kappa} + \frac{\alpha^2}{\kappa}.
\]
where $\phi^* = \argmax_{\pi' \in \widetilde{\Pi}} V^{(p)}(\pi')$.
\end{lemma}

\begin{proof}
Define the function class induced by the disagreements between $\pihat$ and any  $\pi \in \Pi$ 
\[
\cF_\Pi &= \Bigg\{ (x,d, y) \rightarrow \indic{\pi(x)\neq \pihat(x)}  \l(\frac{1}{2}\l( \frac{\kappa yd}{p_o(x)}+\frac{\kappa y(1-d)}{1-p_o(x)}\r)+ p\r) \\
&~~~~~+ \indic{\pi(x)=\pihat(x)}\l( \pihat(x)\frac{\kappa yd}{p_o(x)}+(1-\pihat(x))\frac{\kappa y(1-d)}{1-p_o(x)}\r) : \pi\in\Pi \Bigg\}.
\]
Notice that the class $\mathcal{H}= \l\{\indic{\pi(x)=\pihat(x)}:\pi\in\Pi\r\}$ has VC dimension $d$ ($\pihat$ is a fixed function). We invoke \cref{lem:main_conc}  
with $\gamma_1(z) = \frac{1}{2}\l( \frac{\kappa yd}{p_o(x)}+\frac{\kappa y(1-d)}{1-p_o(x)}\r)+ p$ and $\gamma_2(z) = \pihat(x)\frac{\kappa yd}{p_o(x)}+(1-\pihat(x))\frac{\kappa y(1-d)}{1-p_o(x)} $ and $B = 2$.  Moreover, we have for any $f_{\pi_1}, f_{\pi_2} \in \cF_\Pi$ we have
\[
\E \abs{ f_{\pi_1} - f_{\pi_2} }&\le 2\; \E [\indic{\pi_1 \neq \pi_2} ] \tag{since $\frac{\kappa yd}{p_o(x)} \, , \, \frac{\kappa y(1-d)}{1-p_o(x)} \, , \, p \leq 1 $}\\
&= 2 \; \norm{{\pi_1}-{\pi_2}}{P_X, 1} \leq 2\cD(\Pi)
\]
Plugging this upper bound in the bound in \cref{lem:main_conc} gives us the required bound. 
\end{proof}

\begin{lemma} \label{lem:almost_optimal}
    With probability greater than $1-\delta$ any $\pi \in \Pihat$ satisfies 
    \[
        V(\pist) - V(\pi) \lesssim \frac{1}{\kappa}\l(\alpha \sqrt{\cD(\Pihat)} + \alpha^2\r)
    \]
    where $\cD(\Pihat) =  \sup_{\pi_1,\pi_2} \norm{\pi_1 - \pi}{P_X, 1} $.
\end{lemma}

\begin{proof}
For each $\pi\in\Pihat$, we obtain
    \begin{align}
            V(\pi^*) - V(\pi)
            &= V(\pist) - V(\pi) - \l(V_n(\pist) - V_n(\pi)\r) + \l(V_n(\pist) - V_n(\pi)\r) \nn \\
            &\lesssim \frac{\alpha}{\kappa} \sqrt{\E\abs{f_\pist - f_\pi}} + \frac{\alpha^2}{\kappa} + \l(V_n(\pist) - V_n(\pi)\r) &\text{(via \cref{lem:conc1})} \nn \\
            &\lesssim \frac{\alpha}{\kappa} \sqrt{\E\abs{f_\pist - f_\pi}} + \frac{\alpha^2}{\kappa} + \l(V_n(\pihat) - V_n(\pi)\r) &\text{(since $\pihat$ is the empirical maximizer)} \nn \\
            &\lesssim \tfrac{1}{\kappa}\l(\alpha \sqrt{\E\abs{f_\pist - f_\pi}} + \alpha \sqrt{\E_n\abs{f_\pihat - f_\pi}} + \alpha^2\r) \nn\\
            &\lesssim \tfrac{1}{\kappa}\l(\alpha \sqrt{\E\abs{f_\pist - f_\pi}} + \alpha \sqrt{\E\abs{f_\pihat - f_\pi}} + \alpha^2\r) &\text{(via \cref{lem:conc2})} \nn\\
            &\lesssim \tfrac{1}{\kappa}\l(\alpha \sqrt{\norm{\pist - \pi }{P_X,1}} + \alpha \sqrt{\norm{\pihat - \pi }{P_X,1}} + \alpha^2\r) \nn\\
            &\lesssim \tfrac{1}{\kappa}\l(\alpha \sqrt{\cD(\Pihat)} + \alpha^2\r).\nn 
    \end{align}
\end{proof}


\begin{proof}[Proof of \cref{thm:abstention-rate} ]

We first show that $\pi^* \in \Pihat$. With probability at least $1 - \delta$, we have
    \begin{align*}
        V_n(\pihat) - V_n(\pist)
        &\leq V_n(\pihat) - V_n(\pist) - \l( V(\pihat) - V(\pist) \r) \\
        &\leq  \alpha \sqrt{\E\abs{f_\pist - f_\pihat}} + \alpha^2 \quad \tag{via  \Cref{lem:conc1}} \\
        &\lesssim \alpha \sqrt{\E_n\abs{f_\pist - f_\pihat}} + \alpha^2 \quad \tag{via  \Cref{lem:conc2}}.
    \end{align*}
    Hence, $\pist \in \Pihat$ for an appropriate constant in  Step \ref{step_abs:almost_optimal}   in \cref{algo:abs}.
    
    For the second EWM on $\Pitil$, let $\phi^* \in \Pitil$ be the policy that maximizes the true value $V^{(p)}$, and let $\phi$ be an abstaining policy constructed from $\pi$ ($\phi(x) = *$ if $\indic{\pi(x) \neq \pihat(x)}$ and $\pihat(x)$ otherwise) where $\pi$ satisfies  
    \[
    \norm{\pi - \pihat}{P_X,1} \geq \frac{\cD(\Pihat)}{2}.
    \]
    Such a $\pi$ exists by the definition of $\cD(\Pihat)$ and the triangle inequality. Since $\pi$ and $\pihat$ are binary, by the definition of $f_\pi$ we have
    \begin{equation}
        \notag
        \E\abs{f_\pi-f_\pihat} = \kappa \E\abs{ (\pi(X)-\pihat(X))\left(\frac{YD}{p_o(x)}-\frac{Y(1-D)}{1-p_o(x)}\right)}
        = \kappa \E\l[ \indic{\pi(X) \neq \pihat(X)}\l( \frac{YD}{p_o(x)}+\frac{Y(1-D)}{1-p_o(x)}\r)\r].
    \end{equation}
    Thus, 
    \begin{align}
        V^{(p)}(\phi^*)
        &\geq V^{(p)}(\phi) \nonumber \\
        & = \E\l[\l(\frac{\pi(X)+\pihat(X)}{2}\r) \frac{YD}{p_o(x)} + \l(1-\frac{\pi(X)+\pihat(X)}{2}\r) \frac{Y(1-D)}{1-p_o(x)} + p \cdot \indic{\pi (X) \neq \pihat (X) } \r] \nn \\
        &\geq \E\l[\l(\frac{\pi(X)+\pihat(X)}{2}\r) \frac{YD}{p_o(x)} + \l(1-\frac{\pi(X)+\pihat(X)}{2}\r) \frac{Y(1-D)}{1-p_o(x)}\r] + p \cdot \norm{\pi - \pihat}{P_X,1}     \nn \\
        &= \frac{V(\pihat)+V(\pi)}{2} + \frac{p}{2} \cD(\Pihat) \nn 
    \end{align}
    Further using \cref{lem:almost_optimal} we have 
    \begin{align}
    V^{(p)}(\phi^*) - V(\pist)&\gtrsim    p ~\cD(\Pihat) - \frac{\alpha}{\kappa}\sqrt{\cD(\Pihat)}+ \frac{\alpha^2}{\kappa}.\nn
    \end{align}
    
Since $\pitil$ maximizes $V_n^{(p)}$, using \Cref{lem:conc3} we get
    \[
    V^{(p)}(\pitil) - V^{(p)}(\phi^*) \gtrsim  - \frac{\alpha\sqrt{\cD(\Pihat)}+\alpha^2}{\kappa},
    \]
    which implies
    \begin{equation}
            V^{(p)}(\pitil) - V(\pist) \gtrsim   p ~\cD(\Pihat) - \frac{\alpha}{\kappa}\sqrt{\cD(\Pihat)}+ \frac{\alpha^2}{\kappa}. \label{eq:final_quadratic}
    \end{equation}

    Minimizing the quadratic with respect to $\sqrt{\cD(\Pihat)}$,
    \[
    V(\pist) - V^{(p)}(\pitil)   \lesssim  \frac{\alpha^2}{\kappa^2 p},
    \]
    Finally, substituting $\alpha \coloneqq  \sqrt{\frac{\,d \log\!\frac{n}{d} \;+\; \log\!\frac{1}{\delta}\,}{n}}$ gives us the desired regret bound. 
\end{proof}

\begin{proof}[Proof of  \cref{prop:p-zero-benchmark}]
    Substituting $p=0$ in \cref{eq:final_quadratic} and noting that $\cD(\Pihat) \leq \const$ for some constant $\const$, proves the result.
\end{proof}

%% file: sections/appendix-dr.tex
\section{Doubly Robust Abstention: Additional Details and Proofs}\label{appendix:dr}

\begin{algorithm}[t]
\caption{Unkown Propensities: DR learner. }
\label{algo:abs_dr}
\begin{small}
\begin{algorithmic}[1]
\STATE \textbf{Input:} Samples $\{(X_i,D_i,Y_i)\}_{i=1}^n$, policy class $\Pi$, overlap $\kappa$, confidence $\delta\in(0,1)$, bonus $p$, VC dimension $d$, Propensity estimate $\hat{p}$, Outcome regression estimate $\hat g$.
\STATE \textbf{Set:} $\alpha_\dr \leftarrow \sqrt{\tfrac{d \log\!\frac{n}{d} + \log\!\frac{1}{\delta}}{n}} + \err_\dr$.
\STATE \textbf{Split:} Partition the samples into two sets of size $n/2$: $\mathcal{D}_1, \mathcal{D}_2$.
\STATE \textbf{EWM:} $\displaystyle
\pihat = \argmax_{\pi \in \Pi} \; \vdrn(\pi)$ \quad (computed on $\mathcal{D}_1$).\label{step_abs_dr:first_ewm}
\STATE \textbf{Select near-optimal policies:}
{\small
\[
\widehat{\Pi} \!\leftarrow \!\Big\{ \pi \in \Pi \!:\;
\vdrn(\pihat) - \vdrn(\pi)
\le \frac{\const}{\kappa}\Big(\alpha_\dr^2 + \alpha_\dr \sqrt{\E_n \abs{ \pihat - \pi }} \Big) \Big\}.
\]}\label{step_abs_dr:almost_optimal}
\STATE \textbf{Abstention projection:} For each $\pi \in \widehat{\Pi}$, define
\[
{\pi'}(X) =
\begin{cases}
\pi(X), & \text{if } \pi(X) = \pihat(X),\\
*, & \text{otherwise},
\end{cases}
\quad
\widetilde{\Pi} \leftarrow \{\, {\pi'} : \pi \in \widehat{\Pi} \,\}.
\]\label{step_abs_dr:abs_projection}
\STATE \textbf{EWM with abstention:} $\displaystyle
\pitil \in \mathop{\mathrm{arg\,max}}_{{\pi}\in \widetilde{\Pi}} \; \vdrnp({\pi})$
\quad (evaluate $\vdrnp$ on $\mathcal{D}_2$).\label{step_abs_dr:second_ewm}
\STATE \textbf{Return} $\pitil$.
\end{algorithmic}
\end{small}
\end{algorithm}

\cref{algo:abs_dr} is identical in structure to the known–propensity version, with two key edits.

(i) \emph{Scores:} every occurrence of the outcome $Y$ in the value computations is replaced by the DR pseudo–outcome
$\wh{\varphi}$ defined below, and we optimize the DR objectives $\vdrn(\cdot)$ and $\vdrnp(\cdot)$ in Steps~\ref{step_abs_dr:first_ewm} and \ref{step_abs_dr:second_ewm}.

(ii) \emph{Radius:} the selection radius incorporates nuisance error,
\[
\alpha_\dr \;\leftarrow\; \sqrt{\tfrac{d\log\!\frac{n}{d}+\log\!\frac{1}{\delta}}{n}} \;+\; \err_\dr,
\]
so the near–optimal set in Step~\ref{step_abs_dr:almost_optimal} uses $\alpha_\dr$ and the  disagreement between with the EWM $\E_n|\pihat-\pi|$. Where $\errdr$ is any upper bound on product error in the nuisance: $\err_{\dr} \; \geq \; \E\!\Big[(\wh p(X)-p_o(X))^2\textstyle\sum_{d=0}^1(\wh g(d,X)-g_o(d,X))^2\Big]^{1/2}$.
In practice, we simply increase the constant in Step \ref{step_abs:almost_optimal} of Algorithm~\ref{algo:abs}; when the nuisance product error is $o_p(n^{-1/2})$, this yields the same rate as Theorem~\ref{thm:abstention-rate}. Such product rates are a common assumption in the doubly robust literature \citep{foster2023orthogonal} and hold for several function classes as discussed in \cite[Section 5 and Appendix E]{foster2023orthogonal}. 
Moreover, a natural choice of $\errdr$ is the $L_4$ error rates:
\[
\|\wh p(X)-p_o(X)\|_{P_X,4}\cdot \sum_{d\in\{0,1\}} \|\wh g(d,X)-g_o(d,X)\|_{P_X,4}.
\]
The minimax $L_4$ error rates are well-understood for many nonparametric classes of interest, such as smooth classes \citep{stone1980optimal,stone1982optimal}, Holder classes \citep{lepskii1992asymptotically}, Besov classes \citep{donoho1998minimax} and convex function classes \citep{guntuboyina2015global}.

\paragraph{Pseudo–outcomes and DR objectives.}
For $d\in\{0,1\}$ and nuisance estimates $\wh g(d,x)$, $\wh p(x)$, define
\[
\wh{\varphi}(x,d,y)
\;:=\;
\wh g(d,x)
\;+\;
\Big(\tfrac{d \cdot D}{\wh p(x)}+\tfrac{(1-d)\cdot (1-D)}{1-\wh p(x)}\Big)\big(y-\wh g(d,x)\big).
\]
Under $g_o(\cdot,x)\in[0,1]$ and $p_o(x)\in[\kappa,1-\kappa]$, the DR value functionals are
\[
\vdr(\pi)&=\E\!\big[\pi(X)\wh\varphi(X,1,Y)+(1-\pi(X))\wh\varphi(X,0,Y)\big],\quad\\
\vdrp(\pi)&=\E\!\Big[\indic{\pi\neq *}\,(\cdot)+\indic{\pi=*}\big(\tfrac{\wh\varphi(X,1,Y)+\wh\varphi(X,0,Y)}{2}+p\big)\Big],
\]

Unlike the IPW case, $\vdrn$ and $\vdrnp$ are, in general, \emph{biased} plug–ins; we control this bias by introducing $\err_{\dr}$ (or an upper bound on it).  The concentration lemmas mirror the known–propensity proofs. The resulting bounds (e.g., \cref{lem:conc_dr1,lem:conc_dr2} and their abstention analog \cref{lem:conc_dr3}) yield the same $n,d,\delta$–dependence as in the IPW case, up to constants.  The bias induced by estimating nuisances in $\vdrn$ and $\vdrnp$ is quantified by \cref{lem:did-V,lem:did-Vp}, which bound the drift from $V$ and $V^{(p)}$ in terms of $\err_{\dr}$ and the policy disagreement.

\subsubsection{Proof of \cref{thm:dr-abstention}}

\begin{lemma}\label{lem:conc_dr1}
For any policy $\pi$, with probability at least $1-\delta$, 
\[
\l| V_\dr(\pi^*) - V_{\dr}(\pi) - \l( V_{\dr, n}(\pi^*) - V_{\dr,n}(\pi) \r) \r| \leq \frac{\const}{\kappa} \l( \alpha \sqrt{\E\l| \pi - {\pi^*} \r|} + \alpha^2 \r).
\]
\end{lemma}

\begin{proof}
Define the DR score with fixed nuisances (or computed on a disjoint fold):
\[
f^{\dr}_{\pi}(Z) \;\coloneqq\; \pi(X)\,\wh\varphi(X,1,Y) + \big(1-\pi(X)\big)\,\wh\varphi(X,0,Y),
\]
so that $V_\dr(\pi)=\E[f^{\dr}_{\pi}(Z)]$ and $V_{\dr,n}(\pi)=\E_n[f^{\dr}_{\pi}(Z)]$.  
Under bounded outcomes $Y\in[0,1]$ and strict overlap $p_o(X)\in[\kappa,1-\kappa]$, each pseudo-outcome is a.s.\ bounded:
\[
\big|\wh\varphi(X,1,Y)\big|,\ \big|\wh\varphi(X,0,Y)\big| \;\le\; \frac{C}{\kappa}
\quad\text{a.s.}
\]
 set $B\coloneqq 2/\kappa$. Apply \cref{lem:main_conc} to the class
\[
\cF_\Pi \;=\; \Big\{\, z \mapsto \pi(x)\,\wh\varphi(x,1,y)+\big(1-\pi(x)\big)\,\wh\varphi(x,0,y):\ \pi\in\Pi \,\Big\},
\]
with $\gamma_1=\wh\varphi(\cdot,1,\cdot)$, $\gamma_2=\wh\varphi(\cdot,0,\cdot)$ and the fixed $f^{\dr}_{\pi^*}$. We obtain, with probability at least $1-\delta$,
\[
\Big|\, \E\big[f^{\dr}_{\pi^*}-f^{\dr}_{\pi}\big] - \E_n\big[f^{\dr}_{\pi^*}-f^{\dr}_{\pi}\big] \,\Big|
\;\le\; \const \Big( \alpha\sqrt{B\,\E\!\big| f^{\dr}_{\pi^*}-f^{\dr}_{\pi}\big|} + B\alpha^2 \Big).
\]
Since $\big| f^{\dr}_{\pi^*}-f^{\dr}_{\pi} \big| \le B\,\big|\pi^*-\pi\big|$ pointwise, we have
$\E| f^{\dr}_{\pi^*}-f^{\dr}_{\pi} | \le B\,\E|\pi^*-\pi|$, and the RHS becomes
\[
  \alpha\sqrt{B^2\,\E|\pi^*-\pi|} + B\alpha^2 
\lesssim \frac{\const}{\kappa}\Big( \alpha\sqrt{\E|\pi^*-\pi|} + \alpha^2 \Big),
\]
which is the claim.
\end{proof}

\begin{lemma}\label{lem:conc_dr2}
With probability at least $1-\delta$, for any $f_{\pi_1},f_{\pi_2}\in\cF_\Pi$ the following inequalities hold:
\begin{align*}
\Big|\, \E_n\big| {\pi_1} - {\pi_2} \big| - \E\big| {\pi_1} - {\pi_2} \big| \,\Big|
&\le \const \Big( \alpha \sqrt{\E\big| {\pi_1} - {\pi_2} \big|} + \alpha^2 \Big),\\
\Big|\, \E_n\big| {\pi_1} - {\pi_2} \big| - \E\big| {\pi_1} - {\pi_2} \big| \,\Big|
&\le \const \Big( \alpha \sqrt{\E_n\big| {\pi_1} - {\pi_2} \big|} + \alpha^2 \Big).
\end{align*}
\end{lemma}

\begin{proof}
Apply \cref{lem:main_conc} to the class
\[
\cF_\Pi^{\mathrm{id}} \;=\; \Big\{\, z \mapsto \pi(x)\cdot 1 + \big(1-\pi(x)\big)\cdot 0 :\ \pi\in\Pi \,\Big\},
\]
with $\gamma_1\equiv 1$, $\gamma_2\equiv 0$, hence $B=1$ and $f_\pi(z)=\pi(x)$.  
Fix $f_{\pi^*}$ with $\pi^*=\pi_2$ and take $f_\pi$ with $\pi=\pi_1$; then
\[
|f_{\pi_1}(Z)-f_{\pi_2}(Z)| \;=\; |\pi_1(X)-\pi_2(X)| \;=\; \indic{\pi_1(X)\neq \pi_2(X)}.
\]
The two displayed bounds follow directly from \cref{lem:main_conc} (first with population square-root term, then with the empirical one).
\end{proof}

\begin{lemma}
    For any policy $\pi$, with probability at least $1-\delta$, we have
\[
\l| V_\dr(\pi^*) - V_\dr(\pi) - \l( V_{\dr,n}(\pi^*) - V_{\dr,n}(\pi) \r) \r| \leq \frac{\const}{\kappa} \l( \alpha \sqrt{\E_n\l| \pi - {\pi^*} \r|} + \alpha^2 \r).
\]
\end{lemma}
\begin{proof}
Combine \cref{lem:conc_dr1} with \cref{lem:conc_dr2} by replacing the population square-root term $\sqrt{\E|\pi-\pi^*|}$ in \cref{lem:conc_dr1} with its empirical counterpart via the second inequality of \cref{lem:conc_dr2}.
\end{proof}

Recall the definition, 
\[
\errdr \;\coloneqq\;
\Bigg\{\E\!\Big[\big(\hat p(X)-p_o(X)\big)^2\;\sum_{d\in\{0,1\}}\big(\hat g(d,X)-g_o(d,X)\big)^2\Big]\Bigg\}^{1/2}.
\]

\begin{lemma}\label{lem:did-V}
For any two binary policies $\pi_1,\pi_2\in\Pi$,
\[
\big(V_{\dr}(\pi_1)-V_{\dr}(\pi_2)\big)\;-\;\big(V(\pi_1)-V(\pi_2)\big)
\;\le\; 2\,\kappa^{-1}\,\errdr\;\|\pi_1-\pi_2\|_{P_X,2}.
\]
\end{lemma}

\begin{proof}
By the definitions of $V_{\dr}$ and $V$,
\begin{align*}
V_{\dr}(\pi)
&= \E\!\Big[\pi(X)\!\Big(\hat g(1,X)+\tfrac{\indic{D=1}}{\hat p(X)}\big(Y-\hat g(1,X)\big)\Big)
\\[-0.25em]&\hspace{3.2em}
+\ (1-\pi(X))\!\Big(\hat g(0,X)+\tfrac{\indic{D=0}}{1-\hat p(X)}\big(Y-\hat g(0,X)\big)\Big)\Big] \\
&= \E\!\Big[\pi(X)\!\Big(\hat g(1,X)+\tfrac{p_o(X)}{\hat p(X)}\big(g_o(1,X)-\hat g(1,X)\big)\Big) \\
&\hspace{3.2em}
+\ (1-\pi(X))\!\Big(\hat g(0,X)+\tfrac{1-p_o(X)}{1-\hat p(X)}\big(g_o(0,X)-\hat g(0,X)\big)\Big)\Big],\\
V(\pi) &= \E\!\big[\pi(X)g_o(1,X)+(1-\pi(X))g_o(0,X)\big].
\end{align*}
Subtracting the two displays and taking the difference between $\pi_1$ and $\pi_2$ gives
\begin{align*}
&\big(V_{\dr}(\pi_1)-V_{\dr}(\pi_2)\big) - \big(V(\pi_1)-V(\pi_2)\big) \\
&= \E\Big[(\pi_1-\pi_2)(X)\Big(1-\tfrac{p_o(X)}{\hat p(X)}\Big)\big(g_o(1,X)-\hat g(1,X)\big) \\
&\qquad\qquad\ + (\pi_2-\pi_1)(X)\Big(1-\tfrac{1-p_o(X)}{1-\hat p(X)}\Big)\big(g_o(0,X)-\hat g(0,X)\big)\Big].
\end{align*}
Taking absolute values and using the triangle inequality,
\begin{align*}
\cdots
&\le \E\Big[|\pi_1-\pi_2|(X)\,\Big|1-\tfrac{p_o(X)}{\hat p(X)}\Big|\,\big|g_o(1,X)-\hat g(1,X)\big|\Big] \\
&\quad + \E\Big[|\pi_1-\pi_2|(X)\,\Big|1-\tfrac{1-p_o(X)}{1-\hat p(X)}\Big|\,\big|g_o(0,X)-\hat g(0,X)\big|\Big].
\end{align*}
By strict overlap, $\big|1-\tfrac{p_o}{\hat p}\big|\le \kappa^{-1}|\hat p-p_o|$ and $\big|1-\tfrac{1-p_o}{1-\hat p}\big|\le \kappa^{-1}|\hat p-p_o|$. Hence,
\begin{align*}
\cdots
&\le \kappa^{-1}\,\E\!\Big[|\pi_1-\pi_2|(X)\,|\hat p-p_o|\,\big(|g_o(1,X)-\hat g(1,X)|+|g_o(0,X)-\hat g(0,X)|\big)\Big] \\
&\le \kappa^{-1}\,\|\pi_1-\pi_2\|_{P_X,2}\;
\Big(\E\big[(\hat p-p_o)^2\big(|g_o(1,X)-\hat g(1,X)|+|g_o(0,X)-\hat g(0,X)|\big)^2\big]\Big)^{1/2} \\
&\le 2\,\kappa^{-1}\,\|\pi_1-\pi_2\|_{P_X,2}\;\errdr,
\end{align*}
using Cauchy–Schwarz and $(a+b)^2\le 2(a^2+b^2)$. This proves the claim.
\end{proof}

\begin{lemma}\label{lem:conc_dr3}
Let $\Pi$ be a class of binary-valued functions with VC dimension $d$, and fix $\pihat\in\Pi$. Construct the $\{0,1,*\}$-valued class
\[
 \pi'(x)
 \;=\;
 \begin{cases}
   \pihat(x), & \text{if } \pihat(x)=\pi(x),\\
   *,         & \text{otherwise},
 \end{cases}
 \qquad
 \widetilde{\Pi} \;=\; \{\pi' : \pi\in\Pi\}.
\]
Let $\cD(\Pi)\coloneqq \sup_{\pi_1,\pi_2\in\Pi}\|\pi_1-\pi_2\|_{P_X,1}$. Then, with probability at least $1-\delta$, for every $\pi'\in\widetilde{\Pi}$,
\[
\big|\,V_{\dr}^{(p)}(\phi^*) - V_{\dr}^{(p)}(\pi') \;-\; \big(V_{\dr,n}^{(p)}(\phi^*) - V_{\dr,n}^{(p)}(\pi')\big)\,\big|
\;\lesssim\;
\frac{1}{\kappa}\Big(\alpha\,\sqrt{\cD(\Pi)}+\alpha^2\Big),
\]
where $\phi^* \in \argmax_{\pi' \in\widetilde{\Pi}} V_{\dr}^{(p)}(\pi')$.
\end{lemma}

\begin{proof}
Define the function class induced by the $\pitil$ (disagreement w.r.t.\ $\pihat$):
\[
\cF_\Pi^{\dr}
\;=\;
\Big\{\, (x,d,y)\mapsto
\indic{\pi(x)\neq \pihat(x)}&\Big(\tfrac{1}{2}\big(\wh\varphi(x,1,y)+\wh\varphi(x,0,y)\big)+p\Big)
\\
&+\indic{\pi(x)=\pihat(x)}\big(\pihat(x)\wh\varphi(x,1,y)+(1-\pihat(x))\wh\varphi(x,0,y)\big)
:\ \pi\in\Pi\Big\},
\]
Since $\pihat$ is fixed, the indicator class
$\mathcal{H}=\{\indic{\pi(x)=\pihat(x)}:\pi\in\Pi\}$ has VC dimension $d$. Under bounded outcomes and strict overlap, the DR
scores are a.s.\ bounded by a constant factor of $1/\kappa$, so for any $f_{\pi_1},f_{\pi_2}\in\cF_\Pi^{\dr}$,
\[
\E\big|f_{\pi_1}-f_{\pi_2}\big|
\;\le\; \const\,\E\big[\indic{\pi_1(X)\neq\pi_2(X)}\big]
\;=\; \const \,\|\pi_1-\pi_2\|_{P_X,1}
\;\le\; \const \,\cD(\Pi),
\]
Applying \cref{lem:main_conc} to $\cF_\Pi^{\dr}$
(with the fixed comparator $f_{\phi^*}$ and boundedness constant folded into $1/\kappa$) yields, uniformly over $\pi'\in\widetilde{\Pi}$,
\[
\big|\,\E f_{\phi^*}-\E f_{\pi'} \;-\; (\E_n f_{\phi^*}-\E_n f_{\pi'})\,\big|
\;\lesssim\; \frac{1}{\kappa}\Big(\alpha \sqrt{\E|f_{\phi^*}-f_{\pi'}|}+\alpha^2\Big).
\]
Using  $\E|f_{\phi^*}-f_{\pi'}|\lesssim \cD(\Pi)$ and identifying
$\E f_{\pi'}=V_{\dr}^{(p)}(\pi')$, $\E_n f_{\pi'}=V_{\dr,n}^{(p)}(\pi')$ completes the proof.
\end{proof}
\Cref{lem:did-V} implies that any near-optimal policy under $V_{\dr}$ must also be near-optimal under $V$.

\begin{lemma}\label{lem:did-Vp}
Let $\pi'_1,\pi'_2$ be abstaining policies obtained using disagreement with respect to a fixed reference $\pihat$, from binary policies $\pi_1,\pi_2\in\Pi$. Then
\[
\abs{\big(\vdrp(\pi'_1)-\vdrp(\pi'_2)\big)\;-\;\big(\vp(\pi'_1)-\vp(\pi'_2)\big)}
\;\le\; \kappa^{-1}\,\errdr\;\|\pi_1-\pi_2\|_{P_X,2}.
\]
\end{lemma}

\begin{proof}
For any abstaining policy $\pi'$, define
\[
\Delta^{(p)}(\pi') \;\coloneqq\; \vdrp(\pi')-\vp(\pi').
\]
Condition on $X=x$. The DR–minus–truth discrepancy at $x$ is a linear combination of the two nuisance errors with coefficients determined by the action selected by $\pi'(x)$:
\[
\Delta^{(p)}(\pi' \mid X=x)
=\begin{cases}
\;\Big(1-\tfrac{p_o(x)}{\hat p(x)}\Big)\big(g_o(1,x)-\hat g(1,x)\big), & \pi'(x)=1,\\[0.25em]
\;\Big(1-\tfrac{1-p_o(x)}{1-\hat p(x)}\Big)\big(g_o(0,x)-\hat g(0,x)\big), & \pi'(x)=0,\\[0.25em]
\;\tfrac{1}{2}\!\left[\Big(1-\tfrac{p_o(x)}{\hat p(x)}\Big)\big(g_o(1,x)-\hat g(1,x)\big)
+\Big(1-\tfrac{1-p_o(x)}{1-\hat p(x)}\Big)\big(g_o(0,x)-\hat g(0,x)\big)\right], & \pi'(x)=*.
\end{cases}
\]
By strict overlap,
\[
\Big|1-\tfrac{p_o}{\hat p}\Big|,\ \Big|1-\tfrac{1-p_o}{1-\hat p}\Big| \;\le\; \kappa^{-1}\,|\hat p-p_o|.
\]
Hence, for any $x$,
\[
\big|\Delta^{(p)}(\pi'_1\!\mid X\!=\!x)-\Delta^{(p)}(\pi'_2\!\mid X\!=\!x)\big|
\;\le\; \kappa^{-1}\,|\hat p(x)-p_o(x)|\,\tfrac{1}{2}\!\sum_{d\in\{0,1\}}\big|g_o(d,x)-\hat g(d,x)\big|\;\cdot\;\mathbf{1}\{\pi'_1(x)\neq \pi'_2(x)\}.
\]
Taking expectations and applying Cauchy–Schwarz,
\begin{align*}
\E\Big[\big|\Delta^{(p)}(\pi'_1)-\Delta^{(p)}(\pi'_2)\big| \Big]
&\le \kappa^{-1}\;\E\!\Big[ \indic{\pi'_1 \neq \pi'_2}\;|\hat p-p_o|\;\tfrac{1}{2}\sum_{d}( |g_o(d,X)-\hat g(d,X)| ) \Big] \\
&\le \kappa^{-1}\;\|\indic{\pi'_1 \neq \pi'_2}\|_{P_X,2}\;
\Big(\E\big[(\hat p-p_o)^2 \big(\tfrac{1}{2}\sum_{d}|g_o(d,X)-\hat g(d,X)|\big)^2\big]\Big)^{1/2} \\
&\le \kappa^{-1}\;\|\indic{\pi'_1 \neq \pi'_2}\|_{P_X,2}\;\errdr.
\end{align*}
Finally, by the disagreement–projection construction, $\indic{\pi'_1 \neq \pi'_2} =  \indic{\pi_1\neq \pi_2} = \abs{\pi_1 - \pi_2}$. This yields the stated bound.
\end{proof}


\begin{lemma}\label{lem:did-Vp-star}
Fix a reference policy $\pihat\in\Pi$ and any binary policy $\pi\in\Pi$. Let $\pi'$ be the abstention projection of $\pi$ w.r.t.\ $\pihat$, i.e.,
\[
\pi'(x)=
\begin{cases}
\pi(x), & \pi(x)=\pihat(x),\\
*,      & \text{otherwise}.
\end{cases}
\]
Then
\[
\abs{\big(\,V^{(p)}_{\dr}(\pi')-V_{\dr}(\pi^*)\,\big)\;-\;\big(\,V^{(p)}(\pi')-V(\pi^*)\,\big)}
\;\le\;
\kappa^{-1}\,\err_{\dr}\,\Big(\,\|\pi-\pi^*\|_{P_X,2}+\|\pihat-\pi^*\|_{P_X,2}\Big).
\]
\end{lemma}

\begin{proof}
Define the DR--truth discrepancies
\(
\Delta^{(p)}(\tilde\pi):=V^{(p)}_{\dr}(\tilde\pi)-V^{(p)}(\tilde\pi)
\)
and
\(
\Delta(\tilde\pi):=V_{\dr}(\tilde\pi)-V(\tilde\pi).
\)
Introduce the projection of $\pi^*$ onto $\{0,1,*\}$ via $\pihat$, denoted $\pi^{*'}$, by
\[
\pi^{*'}(x)=
\begin{cases}
\pi^*(x), & \pi^*(x)=\pihat(x),\\
*,        & \text{otherwise}.
\end{cases}
\]
Add and subtract $V^{(p)}_{\dr}(\pi^{*'})$ and $V^{(p)}(\pi^{*'})$:
\begin{align*}
\big(V^{(p)}_{\dr}(\pi')-V_{\dr}(\pi^*)\big)-\big(V^{(p)}(\pi')-V(\pi^*)\big)
&=\underbrace{\Big(\Delta^{(p)}(\pi')-\Delta^{(p)}(\pi^{*'})\Big)}_{\text{(A)}}
\;+\;\underbrace{\Big(\Delta^{(p)}(\pi^{*'})-\Delta(\pi^*)\Big)}_{\text{(B)}}.
\end{align*}
For (A), apply Lemma~\ref{lem:did-Vp} to the pair $(\pi',\pi^{*'})$, which are the abstention projections of $(\pi,\pi^*)$ w.r.t.\ the same $\pihat$:
\[
\text{(A)} \;\le\; \kappa^{-1}\,\err_{\dr}\,\|\pi-\pi^*\|_{P_X,2}.
\]
For (B), $\pi^{*'}$ and $\pi^*$ differ only on $\{x:\pi^*(x)\neq\pihat(x)\}$. We again invoke \cref{lem:did-Vp} but we set $\pist$ at the reference policy  (denoted as $\pihat$ in \cref{lem:did-Vp}), $\pi_1 = \pihat$, $\pi_2 = \pist$ and note that $\pist$ never disagrees with itself and hence  ($\vp(\pi_2') = V(\pi_2)$). Hence we get, 
\[
\text{(B)} \;\le\; \kappa^{-1}\,\err_{\dr}\,\|\pihat-\pi^*\|_{P_X,2}.
\]
Summing the two bounds gives the claim.
\end{proof}

\begin{lemma}\label{lem:pistart_in}
With probability at least $1-\delta$, the optimal policy $\pi^*$ belongs to $\widehat\Pi$ constructed in Step~\ref{step_abs_dr:almost_optimal}.
\end{lemma}

\begin{proof}
Recall $\pihat=\arg\max_{\pi\in\Pi} V_{\dr,n}(\pi)$. By \cref{lem:conc_dr1} and \cref{lem:did-V}, for any $\pi\in\Pi$,
\begin{align*}
V(\pi^*)-V(\pihat)
&\le \big(V_{\dr}(\pi^*)-V_{\dr}(\pihat)\big)
    + \const\,\kappa^{-1}\Big(\alpha \sqrt{\E|\pi-\pi^*|}+\alpha^2\Big) \\
&\le \big(V_{\dr,n}(\pi^*)-V_{\dr,n}(\pihat)\big)
   + \const\,\kappa^{-1}\Big(\alpha \sqrt{\E|\pi-\pi^*|}+\alpha^2\Big)
   + \const\,\kappa^{-1}\,\err_\dr\,\|\pi^*-\pi\|_{P_X,2}.
\end{align*}
Since $\pi,\pi^*$ are binary, $\|\pi^*-\pi\|_{P_X,2}=\sqrt{\E|\pi-\pi^*|}$. Hence
\begin{align}
V(\pi^*)-V(\pihat)
&\le \big(V_{\dr,n}(\pi^*)-V_{\dr,n}(\pihat)\big)
   + \const\,\kappa^{-1}\Big((\err_\dr+\alpha)\sqrt{\E|\pi-\pi^*|}+\alpha^2\Big).
\label{eq:dr-gap}
\end{align}
Because $V(\pi^*)-V(\pihat)\ge 0$, rearranging \eqref{eq:dr-gap} gives
\begin{align*}
V_{\dr,n}(\pihat)-V_{\dr,n}(\pi^*)
&\lesssim \kappa^{-1}\Big((\err_\dr+\alpha)\sqrt{\E|\pi-\pi^*|}+\alpha^2\Big) \\
&\lesssim \kappa^{-1}\Big((\err_\dr+\alpha)\sqrt{\E_n|\pi-\pi^*|}+\err_\dr^2+\alpha^2\Big)
\qquad\text{(by \cref{lem:conc_dr2})}\\
&\lesssim \kappa^{-1}\Big(\alpha_\dr\sqrt{\E_n|\pi-\pi^*|}+\alpha_\dr^2\Big),
\end{align*}
where $\alpha_\dr=\sqrt{\tfrac{d\log\frac{n}{d}+\log\frac{1}{\delta}}{n}}+\err_\dr$. Choosing the constant in Step~\ref{step_abs_dr:almost_optimal} accordingly shows that $\pi^*\in\widehat\Pi$.
\end{proof}

\begin{lemma}\label{lem:almost_optimal_dr}
With probability at least $1-\delta$, every $\pi\in\Pihat$ satisfies
\[
V_{\dr}(\pi^*)-V_{\dr}(\pi)
\;\lesssim\;
\frac{1}{\kappa}\Big(\alpha_{\dr}\sqrt{\cD(\Pihat)}+\alpha_{\dr}^2\Big),
\]
where $\cD(\Pihat)\coloneqq \sup_{\pi_1,\pi_2\in\Pihat}\|\pi_1-\pi_2\|_{P_X,1}$.
\end{lemma}

\begin{proof}
Fix $\pi\in\Pihat$. By \cref{lem:conc_dr1}, on a $1-\delta$ event,
\[
V_{\dr}(\pi^*)-V_{\dr}(\pi)
\;\le\;
\big(V_{\dr,n}(\pi^*)-V_{\dr,n}(\pi)\big)
+ \frac{\const}{\kappa}\Big(\alpha\sqrt{\E|\pi-\pi^*|}+\alpha^2\Big).
\]
Since $\pihat$ maximizes $V_{\dr,n}$, $V_{\dr,n}(\pi^*)-V_{\dr,n}(\pihat)\le 0$, hence
\[
V_{\dr}(\pi^*)-V_{\dr}(\pi)
\;\le\;
\big(V_{\dr,n}(\pihat)-V_{\dr,n}(\pi)\big)
+ \frac{\const}{\kappa}\Big(\alpha\sqrt{\E|\pi-\pi^*|}+\alpha^2\Big).
\]
By the definition of $\Pihat$ (Step~\ref{step_abs_dr:almost_optimal}),
\[
V_{\dr,n}(\pihat)-V_{\dr,n}(\pi)
\;\le\; \frac{\const}{\kappa}\Big(\alpha_{\dr}^2+\alpha_{\dr}\sqrt{\E_n|\pihat-\pi|}\Big).
\]
Combining the last two displays,
\[
V_{\dr}(\pi^*)-V_{\dr}(\pi)
\;\le\;
\frac{\const}{\kappa}\Big(\alpha_{\dr}^2+\alpha_{\dr}\sqrt{\E_n|\pihat-\pi|}
+\alpha\sqrt{\E|\pi-\pi^*|}+\alpha^2\Big).
\]
Because $\pi,\pihat,\pi^*\in\Pihat$, we have $\E|\pihat-\pi|\le \cD(\Pihat)$ and
$\E|\pi-\pi^*|\le \cD(\Pihat)$. Using \cref{lem:conc_dr2}, the fact that
$|\pihat-\pi|\in[0,1]$, and the bound
$\sqrt{a+b}\le \sqrt{a}+\sqrt{b}$, we obtain
\[
\sqrt{\E_n|\pihat-\pi|}
\;\le\; \sqrt{\E|\pihat-\pi|} + \const\,\alpha
\;\le\; \sqrt{\cD(\Pihat)}+\const\,\alpha.
\]
Since $\alpha\le \alpha_{\dr}$, we conclude
\[
V_{\dr}(\pi^*)-V_{\dr}(\pi)
\;\lesssim\;
\frac{1}{\kappa}\Big(\alpha_{\dr}\sqrt{\cD(\Pihat)}+\alpha_{\dr}^2\Big),
\]
as claimed.
\end{proof}

\begin{proof}[Proof of \cref{thm:dr-abstention}]

Let $\phi^*\in\Pitil$ be the policy that maximizes $\vdrp(\cdot)$. We follow a similar proof idea to the known propensity proof. Let $\phi$ be a policy such that $\|\pi-\pihat\|_{P_X,1} \geq \tfrac {\mathcal{D}(\Pihat)}{2}$ . We have that
\begin{equation}
    \notag
    \begin{aligned}
        V_{\dr}^{(p)}(\phi^*) &\geq V_{\dr}^{(p)}(\phi) \\
        &= \E\bigg[ \frac{\pi(X)+\pihat(X)}{2}\hat{\varphi}(X,1,Y) + \left(1-\frac{\pi(X)+\pihat(X)}{2} \right)  \hat{\varphi}(X,0,Y)  +  p\E[ \indic {\pi(X) \neq \pihat(X) }]\bigg] \\
        &\geq \frac{1}{2}\big(V_{\dr}^{(p)}(\pi)+V_{\dr}^{(p)}(\pihat)\big) + p \; \E \abs{\pi(X)-\pihat(X)} \\
        &= \frac{1}{2}\big(V_{\dr}(\pi)+V_{\dr}(\pihat)\big) + p \; \E \abs{\pi(X)-\pihat(X)} & \text{(since $\pi,\pihat$ are binary)} \\
        &\geq \vdr(\pi^*)+ \frac{p}{2} \mathcal{D}(\pihat) - \const \kappa ^ {-1} \left(\alpha_\dr \sqrt{\mathcal{D}(\Pihat)}+\alpha_\dr ^2\right)  \\
    \end{aligned}
\end{equation}

Where the last inequality uses the fact that $\pi^* \in \Pihat$ (\cref{lem:pistart_in}) and \cref{lem:almost_optimal_dr}. Moreover since, $\pitil$ maximizes $\vdrnp$, via \cref{lem:conc_dr3} we have 
\[
\vdrp(\pitil) \gtrsim \vp(\phi^*) - \kappa^{-1} \left(\alpha_\dr \sqrt{\mathcal{D}(\Pihat)}+\alpha_\dr ^2\right) 
\]
Combining with the second last inequality we get 
\[
\vdr(\pi^*) - \vdrp(\pitil) \lesssim   \const \kappa ^ {-1} \left(\alpha_\dr \sqrt{\mathcal{D}(\Pihat)}+\alpha_\dr ^2\right) - \frac{p}{2} \mathcal{D}(\Pihat)
\]
Now we finally invoke \cref{lem:did-Vp-star} to get 
\[
V(\pi^*) - \vp(\pitil) &\lesssim  \const \kappa ^ {-1} \left(\alpha_\dr \sqrt{\mathcal{D}(\Pihat)}+\alpha_\dr ^2\right) + \err_{\dr}\,\Big(\,\|\pi-\pi^*\|_{P_X,2}+\|\pihat-\pi^*\|_{P_X,2}\Big) - \frac{p}{2} \mathcal{D}(\Pihat) \\
& \lesssim \const \kappa ^ {-1} \left(\alpha_\dr \sqrt{\mathcal{D}(\Pihat)}+\alpha_\dr ^2\right) + \err_{\dr} \sqrt{\mathcal{D}(\Pihat)} - \frac{p}{2} \mathcal{D}(\Pihat)
\]
Finally, absorbing the second term into $\alpha_\dr \sqrt{\mathcal{D}(\Pihat)}$, and maximizing the the quadratic wrt $\sqrt{\mathcal{D}(\Pihat)}$ we get 
\[
V(\pi^*) - \vp(\pitil) \lesssim \; \; \frac{\alpha_\dr^2}{\kappa^2 \, p}.
\]
\end{proof}

%% file: sections/appendix-applications.tex
\section{Missing proofs from \cref{sec:applications}}\label{appendix:applications}

\subsection{Proofs of \Cref{thm:finite-D,thm:reg-oracle}}

Recall that the procedure has two phases. In Phase~1 we run \cref{algo:abs} as a black box to obtain an abstaining policy $\pitil$. In Phase~2 we learn a (non-abstaining) policy on the region where $\pitil$ abstains, $\cX_{rem}\subseteq\cX$, either by EWM (when the combinatorial diameter is finite) or via a CATE oracle; denote this policy by $\phi$. The final policy is the mixture
\[
\pi_{\text{final}}(x)=
\begin{cases}
\phi(x), & x\in \cX_{rem},\\
\pitil(x), & x\notin \cX_{rem}.
\end{cases}
\]
For any policy $\pi$, let $V(\pi\mid \cX_{rem})$ and $V(\pi\mid \cX_{rem}^c)$ denote its value contributions on the abstention and non-abstention regions, respectively:
\[
\begin{aligned}
V(\pi; \cX_{rem})
&\coloneqq \E \left[\indic{X\in \cX_{rem}}\big(\pi(X)Y(1)+(1-\pi(X))Y(0)\big)\right],\\
V(\pi; \cX_{rem}^c)
&\coloneqq \E \left[\indic{X\notin \cX_{rem}}\big(\pi(X)Y(1)+(1-\pi(X))Y(0)\big)\right],
\end{aligned}
\]
so that \(
V(\pi)=V(\pi; \cX_{rem})+V(\pi; \cX_{rem}^c).
\)

Hence the value of the final policy is 
\[
V(\pifinal)
&= V(\pitil;\cX_{rem}^c)+V(\phi;\cX_{rem})\\
&= \underbrace{V(\pitil;\cX_{rem}^c)+V(\pi^B;\cX_{rem})}_{\text{(A)}}
\;+\;
\underbrace{V(\phi;\cX_{rem})-V(\pi^B;\cX_{rem})}_{\text{(B)}}.
\]
Where $\pi^B$ is the Bayes optimal policy. Further, we can write (A) as 
\[
\text{(A)} &=  \E\Big[\indic{\pitil(X) \neq *} \l(\pitil(X) Y(1) + (1- \pitil(X)) Y(0)\r)\Big] +  \E\Big[\indic{\pitil(X) = *} \l( \pi^B(X)Y(1) + (1- \pi^B(X)) Y(0) \r)\Big] \\
&= \E\Big[\indic{\pitil(X) \neq *} \l(\pitil(X) Y(1) + (1- \pitil(X)) Y(0)\r)\Big]+ \E\l[\indic{\pitil(X) = *} \l(\frac{Y(1) + Y(0)}{2} + \frac{\abs{\tau_o(X)}}{2} \r)\r] \tag{since $\pi_B(x) = \indic{\tau_o(x) >0}$}\\
&\geq  \E\Big[\indic{\pitil(X) \neq *} \l(\pitil(X) Y(1) + (1- \pitil(X)) Y(0)\r)\Big] + \E\l[\indic{\pitil(X) = *} \l(\frac{Y(1) + Y(0)}{2} + \frac{h}{2} \r)\r] \tag{ margin condition \eqref{eq:margin}}\\
&=V^p(\pitil)
\]
Where $V^{(p)}$ is the abstention value with $p= \frac{h}{2}$.

Hence, we can now write the regret of $\pifinal$ as 
\[
V(\pist) - V(\pifinal) &\leq V(\pist) - \vp(\pitil)- \text{(B)} \\
&\lesssim \frac{d \log \frac{n}{d}+\log \frac{1}{\delta}}{p\,n\,\kappa^2} - \text{(B)} \tag{ via \cref{thm:abstention-rate}}
\]

In the remainder of this section, we will focus on proving bounds for (B). Conditioning on the first phase makes $\cX_{rem}$ fixed (measurable w.r.t.\ $\cD_1\cup\cD_2$), so analysis on the third split (i.e. bounds on (B)) treats it as nonrandom. Finally,  the third split has size $m$ ( $m=n/3$) and with $\rho :=\P(X\in\cX_{rem})$ and $N:=\sum_{i=1}^m \indic {X_i\in\cX_{rem}}\sim\mathrm{Bin}(m,\rho)$,  applying chernoff bound we get
\[
 N \;\ge\; m\rho \;-\; \sqrt{2\,m \rho \,\log \tfrac{1}{\delta}} \quad \text{with probability at least } 1-\delta. 
\]

If $\rho < \tfrac{8 \log(1/\delta)}{m}$, we obtain a trivial bound on (B):

\[
-(B)
&= V(\pi^B;\cX_{rem}) - V(\phi;\cX_{rem})
 = \E \Big[\indic{X \in \cX_{rem}}\,\big(\pi^B(X) - \phi(X)\big)\big(Y(1) - Y(0)\big)\Big] \\
&= \E \Big[\big(\pi^B(X) - \phi(X)\big)\big(Y(1) - Y(0)\big)\,\big|\, X \in \cX_{rem}\Big]\;\P(X\in\cX_{rem}) \\
&\le \P(X\in\cX_{rem}) \tag{bounded $Y$}\\
&=\; \rho \;\le\; \frac{8 \log(1/\delta)}{m} \;\lesssim\; \frac{8 \log(1/\delta)}{n},
\]

in which case the conclusions of \Cref{thm:finite-D,thm:reg-oracle} follow immediately.

Otherwise, when $\rho \ge \tfrac{8 \log(1/\delta)}{m}$, the Chernoff bound implies that with probability at least $1-\delta$,
\begin{equation}
    N \;\ge\; \frac{m\rho}{2} = \frac{n}{6} \, \prob\{ X \in \cXrem \}.  \label{eq:N_large}
\end{equation}

\paragraph{Proof of \cref{thm:finite-D} }
By finite combinatorial diameter $D$, the set of points on which policies in $\Pi$ can disagree has size at most $D$; since $\cX_{rem}$ collects (projected) disagreements w.r.t.\ $\pihat$, we have $|\cX_{rem}|\le D$, hence the number of labelings on $\cX_{rem}$ is $2^{|\cX_{rem}|}\le 2^D$, so we may run EWM over the finite class $\Pi_{rem}$ (all $2^{|\cX_{rem}|}$ labelings)  and, in particular, the Bayes rule restricted to $\cX_{rem}$, $\pi^B(x)= \indic{\tau_o(x)>0}$, belongs to $\Pi_{rem}$. 

\[
 V(\pi^B;\cX_{rem}) - V(\phi;\cX_{rem}) &= \E\,\Big[\indic{X \in \cX_{rem}}\,\big(\pi^B(X) - \phi(X)\big)\big(Y(1) - Y(0)\big)\Big] \\
 &=\E\,\Big[\big(\pi^B(X) - \phi(X)\big)\big(Y(1) - Y(0)\big)\,\big|\, X \in \cX_{rem}\Big]\;\P(X\in\cX_{rem})
\]
Hence we can invoke Theorem 2.3 of \citet{kitagawa2018should}, which establishes fast rates for policy learning under a (soft) margin condition. Our setting imposes a hard margin, which is a special case of theirs (formally, take the soft–margin parameter $\alpha\to\infty$). In fact, for a finite policy class, one may use the Bernstein inequality for each policy and union bound, instead of using uniform bounds for VC classes. We  obtain using \eqref{eq:N_large}:
\[
\E\,\Big[\big(\pi^B(X) - \phi(X)\big)\big(Y(1) - Y(0)\big)\,\big|\, X \in \cX_{rem}\Big] \lesssim \frac{\log{ \tfrac{ 2^D} {\delta}}}{\kappa^2\, h  }\; \frac{1}{\, n \, \P\{X \in \cXrem\}}
\]
Hence, we finally get, 
\[
 V(\pi^B;\cX_{rem}) - V(\phi;\cX_{rem}) \lesssim  \frac{D+ \log \tfrac{ 1} {\delta}} {\kappa^2\, h \, n }
\]

\paragraph{Proof of \cref{thm:reg-oracle}}
We start from the decomposition over the abstention region:
\begin{align}
V(\pi^B;\cX_{rem})-V(\phi;\cX_{rem})
&= \E\Big[\indic{X\in\cX_{rem}}\big(\pi^B(X)-\phi(X)\big)\big(Y(1)-Y(0)\big)\Big]\nonumber\\
&= \E\Big[\big(\pi^B(X)-\phi(X)\big)\big(Y(1)-Y(0)\big)\ \big|\ X\in\cX_{rem}\Big]\ \P(X\in\cX_{rem}) .
\label{eq:reg_main}
\end{align}

Since $\pi^B(x)=\indic{\tau_o(x)>0}$ and $\phi(x)=\indic{\widehat{\tau}(x)>0}$, we can write
\begin{align*}
\E\!\Big[\big(\pi^B(X)-\phi(X)\big)\big(Y(1)-Y(0)\big)\ \big|\ X\in\cX_{rem}\Big]
&= \E\!\Big[\indic{\mathrm{sign}(\tau_o(X))\neq \mathrm{sign}(\widehat{\tau}(X))}\,|\tau_o(X)|\ \big|\ X\in\cX_{rem}\Big].
\end{align*}

By the hard margin assumption $|\tau_o(X)|\ge h$, whenever the signs of $\tau_o$ and $\widehat{\tau}$ disagree, they are bound to be separated by at least $h$ gap.

\[
 \E \bigg[\indic{\sign(\tau_o(X)) \neq \sign(\tauh(X))}& \,\abs{\tau_o(X)} \biggm| X \in \cX_{rem}\bigg]\\
&= \E \bigg[\indic{\sign(\tau_o(X)) \neq \sign(\tauh(X))}\,\indic{\abs{\tauh(X)-\tau_o(X)} \ge h}\,\abs{\tau_o(X)} \biggm| X \in \cX_{rem}\bigg]
\]
Further we write 
\[
 \E \bigg[\indic{\sign(\tau_o(X)) \neq \sign(\tauh(X))}&\,\indic{\abs{\tauh(X)-\tau_o(X)} \ge h}\,\abs{\tau_o(X)} \biggm| X \in \cX_{rem}\bigg] \\
&\leq \E \bigg[\indic{\abs{\tauh(X)-\tau_o(X)} \ge h}\,\abs{\tauh(X)-\tau_o(X)} \biggm| X \in \cX_{rem}\bigg] 
\]
To bound the last expectation, use the elementary inequality 
$ \indic{|u|\ge a}\,|u| \;\le\; \frac{u^2}{a} \text{ for all }u\in\mathbb{R},\ a>0$.
\begin{align}
   \E \bigg[\indic{\abs{\tauh(X)-\tau_o(X)} \ge h}\,\abs{\tauh(X)-\tau_o(X)} \biggm| X \in \cX_{rem}\bigg] &\le \frac{2}{h}\,\E \l[(\tauh(X)-\tau_o(X))^2 \biggm| X \in \cX_{rem}\r]  \nn \\
& \lesssim \frac{c_\delta}{h}  \l(n ~ \prob(X \in {\cal{X}}_{rem})\r)^{-2 \beta}  \label{eq:cond-gap} 
\end{align}
The last inequality follows via \eqref{eq:N_large} and CATE Oracle rate.
Next, we bound the mass of the abstention region. By construction of $\cX_{rem}$ there exists $\pi\in\Pihat$ such that
$\indic{X\in\cX_{rem}}=\indic{\pi(X)\neq \pihat(X)}$.  
\[
\prob \l( X \in \cX_{rem}\r) &=  \E\l[\indic{ \pi(X) \neq \pihat(X)} \r] \\
&= \E \l[ \indic{ \pi(X) \neq \pi^B(X)}\r] + \E \l[ \indic{ \pihat(X) \neq \pi^B(X)}\r] \\
&\leq \frac{1}{h}\l( E\l[ \indic{ \pi(X) \neq \pi^B(X)}h\r] + \E \l[ \indic{ \pihat(X) \neq \pi^B(X)}h\r]\r)\\
& \leq \frac{1}{h} \l( E\l[ \indic{ \pi(X) \neq \pi^B(X)} \abs{\tau_o(X)}\r] + \E \l[ \indic{ \pihat(X) \neq \pi^B(X)}\abs{\tau_o(X)}\r] \r)\\
& =\frac{1}{h} \l(  V(\pi^B) - V(\pi)  +   V(\pi^B) - V(\pihat) \r)\\
 &\lesssim \frac{1}{h} \l(V(\pi^*(X)) - V(\pi^B(X)) + \frac{\alpha}{\kappa} \r)
\]
using $|\tau_o|\ge h$ and that $\pihat, \pi$ belong in the almost optimal policy set $\Pihat$ (so $V(\pi^*)-V(\pihat)\lesssim \alpha/\kappa$ with $\alpha = \frac{1}{\kappa} \sqrt{\frac{d \log{\tfrac{n}{d}} + \log{\tfrac{1}{\delta}}}{n}}$.

Finally, insert \eqref{eq:cond-gap} into \eqref{eq:reg_main} and expand the right–hand side using the bound on $\P(X\in\cX_{rem})$. Collecting the terms arising from (i) $V^{(p)}$ regret, (ii) the oracle CATE estimation on $\cX_{rem}$ and the bound on $\P(X\in\cX_{rem})$ yields the three contributions $\mathrm{Reg}_1,\mathrm{Reg}_2,\mathrm{Reg}_3$ stated in the theorem.

\subsection{Proof of \cref{prop:dist-shift}}

We show that the policy abstention setting that we introduced has a natural relationship with distributionally robust policy learning. Specifically, consider a setting where the true data distribution of $(Y(0),Y(1))\mid X$ is different from our training data. This would happen when our observational data is outdated and cannot reflect the true effect of the current treatment of interest. For instance, we may want to find an optimal subpopulation for up-to-date vaccination, but we only have outcome data for an older version of the vaccine, that might be close to but different from the potential outcomes of the latest vaccine.

While deterministic policies are optimal without such an outcome distribution shift, they can be problematic otherwise. Intuitively, deterministic policies tend to ``put all eggs in one basket'', rendering them more vulnerable to any potential systematic shift of the potential outcomes. To mathematically formulate this intuition, we assume the true potential outcome distribution lies in some $W_1$-ball of the distribution that generates our observations, \emph{i.e.}
\begin{equation}
    \notag
    \prob_{\mathtt{test}} \in \mathcal{P}_{\alpha}(\prob_{\mathtt{train}}) := \left\{\prob: W_1(\prob,\prob_{\mathtt{train}})\leq\alpha\right\}
\end{equation}
and we would like to maximize the worst case value of a policy induced by the ambiguity set $\mathcal{P}_{\alpha}(\prob_{\mathtt{train}})$.

\begin{proof}[Proof of \cref{prop:dist-shift}]
    Fix $X$ and write $\mu_d(x) = \mathbb{E}[Y(d) \mid X = x]$, $d \in \{0,1\}$, and define, for a policy $\tilde\pi: X \to \{0,1,*\}$, the conditional reward at $x$ by 
\begin{align*}
r_{\tilde{\pi}}(x) &= \mathbf{1}\{\tilde{\pi}(x) \neq *\}[\tilde{\pi}(x)\mu_1(x) + (1-\tilde{\pi}(x))\mu_0(x)] \\
&\quad + \mathbf{1}\{\tilde{\pi}(x) = *\} \cdot 0.5(\mu_0(x) + \mu_1(x));
\end{align*}

let $P_\alpha(P_{\text{train}}) = \{P : W_1(P, P_{\text{train}}) \leq \alpha\}$ with the $W_1$ ball taken on the $(Y(0), Y(1))$ coordinates and the $X$-marginal fixed at $P_X^*$, so $V(\pi) = \mathbb{E}_P[r_\pi(X)]$ and 
$$\min_{P \in P_\alpha} V(\pi) = \mathbb{E}_{P_X^*}\left[\min_{P_X(x)} \mathbb{E}_P[r_\pi(X) \mid X = x]\right].$$

By Kantorovich--Rubinstein duality for the $\ell_1$ ground metric on $(Y(0), Y(1))$, for any affine $g(u,v) = au + bv$ we have 
$$\inf_{W_1(P, P_{\text{train}}) \leq \alpha} \mathbb{E}_P[g(Y(0), Y(1))] = \mathbb{E}_{P_{\text{train}}}[g] - \alpha \|(a,b)\|_\infty.$$

Apply this pointwise at $x$ with $g$ chosen according to $\pi(x)$, this becomes
\[
    \pi(x)\E[Y(1)\mid X=x] + (1-\pi(x))\E[Y(0)\mid X=x] - \alpha \max\{\pi(x), 1-\pi(x)\}.
\]

Since this is a linear function on $[0,0.5]$ and $[0.5,1]$, it must attains maximum value in $\{0,0.5,1\}$. If $\pi(x) = 1$ then $g(u,v) = v$ and $\|(0,1)\|_\infty = 1$, if $\pi(x) = 0$ then $g(u,v) = u$ and $\|(1,0)\|_\infty = 1$, and if $\pi(x) = 0.5$ then $g(u,v) = 0.5(u+v)$ and $\|(0.5, 0.5)\|_\infty = 0.5$; hence for every $x$,
\begin{align*}
\min_{P_X(x)} \mathbb{E}_P[r_\pi(X) \mid X = x] &= \mathbf{1}\{\pi(x) \neq 0.5\}[\pi(x)\mu_1(x) + (1-\pi(x))\mu_0(x) - \alpha] \\
&\quad + \mathbf{1}\{\pi(x) = 0.5\}[0.5(\mu_0(x) + \mu_1(x)) - \alpha/2] \\
&= \mathbf{1}\{\tilde{\pi}(x) \neq 0.5\}[\tilde{\pi}(x)\mu_1(x) + (1-\tilde{\pi}(x))\mu_0(x) - \alpha] \\
&\quad + \mathbf{1}\{\tilde{\pi}(x) = *\}[0.5(\mu_0(x) + \mu_1(x)) - \alpha/2]
\end{align*}

Taking expectation over $X$ yields 
\begin{align*}
\min_{P \in P_\alpha} V(\pi) &= \mathbb{E}[\mathbf{1}\{\tilde{\pi}(X) \neq *\}(\tilde{\pi}\mu_1 + (1-\tilde{\pi})\mu_0) + \mathbf{1}\{\tilde{\pi}(X) = *\} \cdot 0.5(\mu_0 + \mu_1)] \\
&\quad - \alpha \cdot \mathbb{P}(\tilde{\pi}(X) \neq *) - (\alpha/2) \cdot \mathbb{P}(\tilde{\pi}(X) = *).
\end{align*}

Comparing with the abstention objective 
$$V^{(p)}(\pi) = \mathbb{E}[\mathbf{1}\{\tilde{\pi} \neq *\}(\pi\mu_1 + (1-\pi)\mu_0) + \mathbf{1}\{\tilde{\pi} = *\}(0.5(\mu_0 + \mu_1) + p)]$$ 
at $p = \alpha/2$ gives $\min_{P \in P_\alpha} V(\pi) = V^{(\alpha/2)}(\pi) - \alpha$, because the difference at each $x$ is $\alpha$ in both the non-abstain case (no bonus vs $-\alpha$) and the abstain case ($+\alpha/2$ vs $-\alpha/2$).
\end{proof}




%% file: sections/appendix-experiments.tex
\section{Experiment Details}\label{appendix:experiments}

In this section, we provide further details of the experiments.

\subsection{Policy Learning with Abstention}

\textbf{Experimental setting.} We study Algorithm \ref{algo:abs} on synthetic policy learning problems with known ground truth. For each problem instance, $X$ are drawn i.i.d. (standard normal), treatment is assigned with an $X$‑dependent logistic propensity in $[0.1, 0.9]$, and observed outcomes are $Y = Y(D) + \epsilon$ that lie in $[0,1]$. We consider three reward regimes (linear, nonlinear, complex), multiple noise levels, and a range of feature dimensions and sample sizes.

\textbf{Policy class and hyperparameters.} The base class $\Pi$ contains simple threshold policies (including linear‑threshold with intercept). We follow Algorithm \ref{algo:abs} and select an empirical‑welfare maximizer from $\Pi$. Algorithm parameters are set to be $\kappa=0.1$, confidence $\delta=0.05$, and abstention bonus $p=0.05$. 

\textbf{Evaluation protocol.} For each configuration we report: (i) Monte‑Carlo ground‑truth $V(p)(\pi)$ computed with many draws; (ii) IPW estimates on observed data; (iii) abstention rate. We run multiple replications per configuration to produce reliable results.

\begin{figure}
    \centering
    \includegraphics[width=0.9\linewidth]{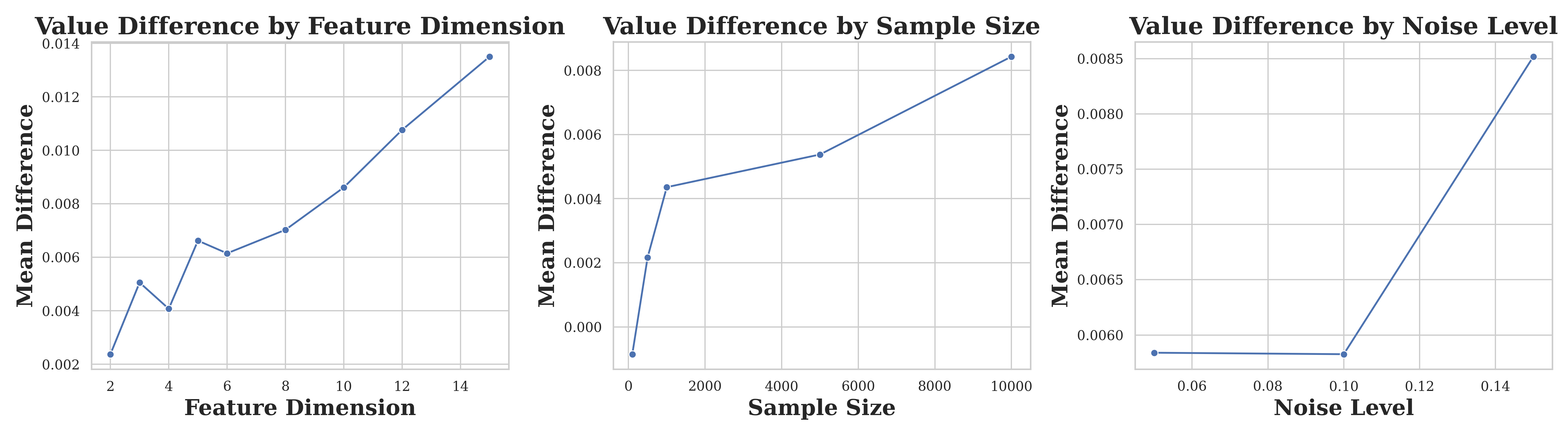}
    \caption{Mean value difference under abstention.}
    \label{fig:abstention-val-diff}
\end{figure}

\begin{figure}
    \centering
    \includegraphics[width=0.9\linewidth]{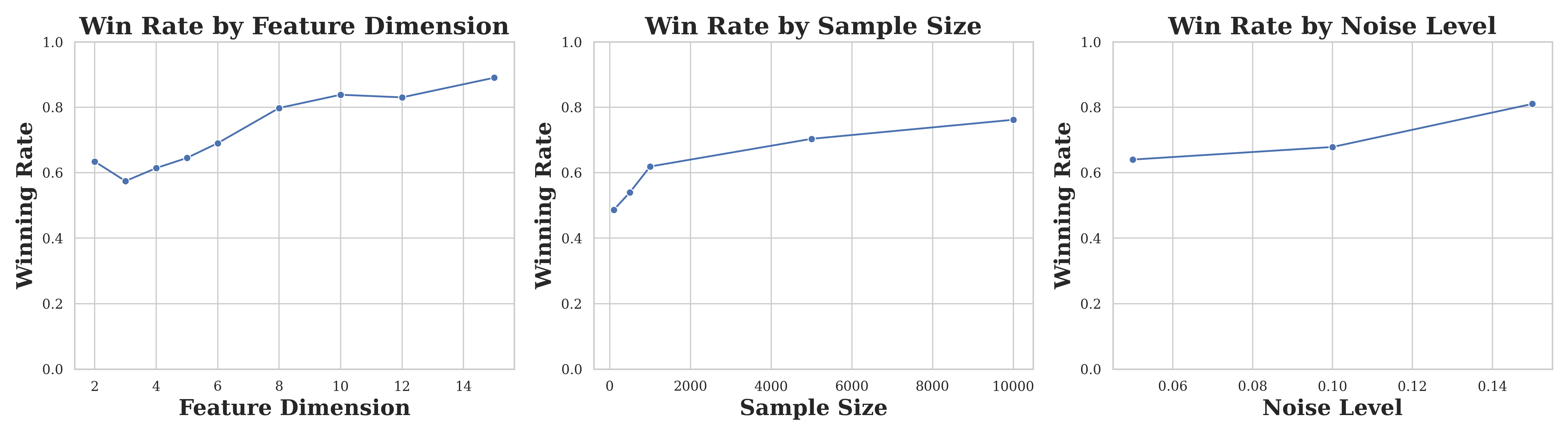}
    \caption{Winning rate of Algorithm \ref{algo:abs} compared with EWM.}
    \label{fig:abstention-win-rate}
\end{figure}

From Figures \ref{fig:abstention-val-diff} and \ref{fig:abstention-win-rate}, we can see that under all parameter configurations, abstention beats EWM in most instances, with positive mean difference relative to EWM, indicating its superiority.

\subsection{Safe Policy Learning}

We evaluate safe policy improvement methods that must avoid degrading a fixed baseline $\omega$. Experiments are conducted under varying reward variance and baseline optimality gap.

\textbf{Data‑generating process.} For each dataset we sample $X \sim U[0,1]^d (d=5)$. The CATE is chosen to be $\tau(X)=2(X_1+X_2-1)$;  and the potential outcomes are $Y(0)=X_3+\epsilon, Y(1)=Y(0)+\tau(X), \epsilon\sim N(0,\sigma^2)$; treatment $D\sim \text{Bernoulli}(p(X))$, with $p(X)$ either constant or logistic in $X$, clipped to $[0.1, 0.9]$. The true optimal policy is $\pi^*(x)=1\{\tau(X)>0\}$.

\textbf{Baselines.} We compare with several baseline algorithms for safe policy learning: (i) Safe EWM: direct EWM followed by comparing its estimated value with that of the baseline policy, and (ii) two  versions ($t$‑test LCB and clipped‑CI) of the HCPI algorithms proposed in \citet{thomas2015high}.

For each parameter value and sample size $n\in \{200, 500, 1000, 2000, 5000\}$, we run 500 independent replications. For each method we compute ground‑truth value using $Y(0),Y(1)$ and report mean true‑value gain $V(\pi_{chosen})-V(\omega)$, rate of mistake $V(\pi_{chosen})<V(\omega)$, and rate of improvement $P[V(\pi_{chosen})>V(\omega)]$. Results are aggregated across runs and visualized as in Figures \ref{fig:varying_noise} and \ref{fig:varying_base_opti_gap} respectively.